\documentclass[10pt,journal,compsoc]{IEEEtran}
\ifCLASSOPTIONcompsoc
  \usepackage[nocompress]{cite}
\else
  \usepackage{cite}
\fi
\ifCLASSINFOpdf
  \usepackage[pdftex]{graphicx}
\else
\fi
\usepackage{amsmath}
\DeclareMathOperator{\st}{s.t. }

\usepackage{algorithm}
\usepackage{algpseudocode}

\algnewcommand\algorithmicinput{\textbf{Input:}}
\algnewcommand\Input{\item[\algorithmicinput]}
\algnewcommand\algorithmicoutput{\textbf{Output:}}
\algnewcommand\Output{\item[\algorithmicoutput]}
\algnewcommand{\LineComment}[1]{\State \(\triangleright\) #1}

\usepackage{booktabs}

\makeatletter
\newcommand\fs@booktabsruled{%
    \def\@fs@cfont{\bfseries\strut}\let\@fs@capt\floatc@ruled
    \def\@fs@pre{\hrule height\heavyrulewidth depth0pt \kern\belowrulesep}%
    \def\@fs@mid{\kern\aboverulesep\hrule height\lightrulewidth\kern\belowrulesep}%
    \def\@fs@post{\kern\aboverulesep\hrule height\heavyrulewidth\relax}%
    \let\@fs@iftopcapt\iftrue
}
\makeatother
\floatstyle{booktabsruled}
\restylefloat{algorithm}

\usepackage{amsthm}
\newtheorem{theorem}{Theorem}
\newtheorem{assumption}{Assumption}
\newtheorem{lemma}{Lemma}

\usepackage{hyperref}
\usepackage[nameinlink,noabbrev]{cleveref}
\usepackage{array, multirow, multicol, threeparttable}

\usepackage{times}
\usepackage{textcomp}
\usepackage{xr}
\usepackage{zref-base}
\usepackage{atveryend}
\makeatletter
\AfterLastShipout{%
    \if@filesw
    \begingroup
    \zref@setcurrent{default}{\the\value{equation}}%
    \zref@wrapper@immediate{%
        \zref@labelbyprops{eq-last}{default}%
    }%
    \endgroup
    \fi
}
\makeatother

\DeclareMathOperator*{\argmin}{arg\,min}
\usepackage{amstext}
\usepackage{mathtools}
\DeclareMathOperator{\tr}{tr}

\DeclareMathOperator{\var}{var}
\newcommand\norm[1]{\left\lVert#1\right\rVert}

\begin{document}
%
\title{Distributed Bayesian Matrix Decomposition for Big Data Mining and Clustering}
%
%
%
%

\author{Chihao~Zhang,~Yang~Yang, ~Wei~Zhou
    and~Shihua~Zhang*
    \IEEEcompsocitemizethanks{\IEEEcompsocthanksitem Chihua Zhang and Shihua Zhang are with the NCMIS, CEMS, RCSDS, Academy of Mathematics and Systems Science, Chinese Academy of Sciences, Beijing 100190, School of Mathematical Sciences, University of Chinese Academy of Sciences, Beijing 100049, and Center for Excellence in Animal Evolution and Genetics, Chinese Academy of Sciences, Kunming 650223, China.
    \IEEEcompsocthanksitem Yang Yang and Wei Zhou are with the School of Software, Yunnan University, Kunming 650504, China.
    \protect\\ *To whom correspondence should be addressed. Email: zsh@amss.ac.cn.}}

%
%

\markboth{ZHANG C, Yang Y., Zhou W., ZHANG S.: DISTRIBUTED BAYESIAN MATRIX DECOMPOSITION FOR BIG DATA MINING AND CLUSTERING}%
{IEEE TRANSACTIONS ON JOURNAL NAME,  MANUSCRIPT ID}
%



\IEEEtitleabstractindextext{%
\begin{abstract}
Matrix decomposition is one of the fundamental tools to discover knowledge from big data generated by modern applications. However, it is still inefficient or infeasible to process very big data using such a method in a single machine. Moreover, big data are often distributedly collected and stored on different machines. Thus, such data generally bear strong heterogeneous noise. It is essential and useful to develop distributed matrix decomposition for big data analytics. Such a method should scale up well, model the heterogeneous noise, and address the communication issue in a distributed system. To this end, we propose a distributed Bayesian matrix decomposition model (DBMD) for big data mining and clustering. Specifically, we adopt three strategies to implement the distributed computing including 1) the accelerated gradient descent, 2) the alternating direction method of multipliers (ADMM), and 3) the statistical inference. We investigate the theoretical convergence behaviors of these algorithms. To address the heterogeneity of the noise, we propose an optimal plug-in weighted average that reduces the variance of the estimation. Synthetic experiments validate our theoretical results, and real-world experiments show that our algorithms scale up well to big data and achieves superior or competing performance compared to other distributed methods.
\end{abstract}

}

\maketitle

\IEEEdisplaynontitleabstractindextext

%
\IEEEpeerreviewmaketitle

\IEEEraisesectionheading{\section{Introduction}\label{sec:introduction}}

%
%
%
%

\IEEEPARstart{B}{ig} data emerge from various disciplines of sciences with the development of technologies. For example, different types of satellite platforms have generated vast amounts of remote sensing data, and high-throughput sequencing technologies provide large-scale transcriptomic data. Such data are usually organized in the matrix form and are generally redundant and noisy. Therefore, matrix decomposition becomes one of the fundamental tools to explore the data. The history of matrix decomposition dates back to more than one century ago when Pearson invented principal component analysis (PCA) \cite{Pearson1901}. Since then, matrix decomposition has been extensively studied due to its effectiveness and is still an active topic today. The conceptual idea of matrix decomposition is that the primitive big and noisy data matrix can be approximated by the product of two or more compact low-rank matrices. Mathematically, given the observed data matrix $\boldsymbol{X} \in R^{m\times n}$, matrix decomposition methods consider
\begin{equation}\label{eq:md}
\min_{\boldsymbol{W},\boldsymbol{H}}\text{D}(\boldsymbol{X}||\boldsymbol{W}\boldsymbol{H}),
\end{equation}
where $\boldsymbol{W}\in R^{m\times r}$, $\boldsymbol{H} \in R^{r \times n}$, and $\text{D}$ is a divergence function.
Typically,  $r \ll \min(m, n)$, and thus $\boldsymbol{W}$ and $\boldsymbol{H}$ are the compact basis and coefficient matrices, respectively.

\par Matrix decomposition methods are flexible by imposing different restrictions or regularizers on $\boldsymbol{W}$, $\boldsymbol{H}$, and choosing different divergence functions. From the view of minimization of the reconstruction error \cite{Bishop2006}, PCA is merely the matrix decomposition method with the Frobenius norm as the divergence function $\norm{\boldsymbol{X} - \boldsymbol{W}\boldsymbol{H}}_F^2$.
It is known that PCA is sensitive to gross errors.
Shen \textit{et al}.\cite{Shen2014} developed a robust PCA by replacing the Frobenius norm with $L_1$-norm.
Min \textit{et al}. \cite{Min2019} imposed group-sparse penalties on SVD to account for prior group effects.
The famous clustering algorithm k-means can also be understood as a matrix decomposition method.
If we choose the Frobenius norm as the divergence function and restrict $\boldsymbol{H}$ such that its each column indicates the cluster membership by a single one, then Eq. (\ref{eq:md}) is exactly the k-means method.
Moreover, it is common that the primitive data are naturally nonnegative, such as images and text corpus.
Hence, the nonnegative matrix factorization (NMF) has been explored \cite{Paatero1994,Lee1999}.
NMF restricts both $\boldsymbol{W}$ and $\boldsymbol{H}$ to be nonnegative. The negativeness enhances the interpretability of the model and leads to part-based feature extraction and sparseness \cite{Lee1999}.
Therefore, NMF gets popular and has many variants developed, including sparse NMF \cite{Hoyer2002,Kim2007} and graph regularized NMF \cite{Cai2008, Cai2011, Zhang2011}.

\par However, the standard matrix decomposition methods are prone to overfitting the observed matrix that is noisy and has missing values. The imposed regularizers can reduce the risk of overfitting, but their parameters require carefully tuning. Bayesian matrix decomposition addresses this problem by incorporating priors into the model parameters.
Tipping and Bishop \cite{Tipping1999} first proposed the probabilistic PCA (PPCA) and showed that PPCA is a latent variable model with independent and identical distribution (IID) Gaussian noise. The probabilistic treatment permits various extensions of PCA \cite{bishop1999bayesian, Collins2001, mohamed2009bayesian, Li2013}.
Analogous to Bayesian PCA, the Bayesian treatments for k-means \cite{Welling2006}, NMF \cite{Schmidt2009, Cemgil2009} have also been explored. In collaborative filtering, Salakhutdinov and Mnih \cite{Salakhutdinov2008} proposed the Bayesian probabilistic matrix factorization (BPMF). One of the appealing characteristics of the Bayesian approach is that it gives the flexibility to design different matrix decomposition methods by choosing appropriate distributions for priors and noise. Saddiki \textit{et al}.\cite{Saddiki2015} proposed a mixed-membership model named GLAD that employs priors of Laplace and Dirichlet distribution on $\boldsymbol{W}$ and columns of $\boldsymbol{H}$, respectively.
Multi-view data collected from different sources are now ubiquitous, and bear distinctly heterogeneous noise \cite{Xu2013}. Such data from different sources are complementary, and thus computational methods for integrative analysis are urgently needed. Some matrix decomposition methods for multi-view data integration have been explored \cite{Zhang2012, Jing2012, Liu2013, Zhang2019b}.
Typically, those algorithms assume that the data matrices share a common basis matrix (or coefficient matrix), enabling the methods to perform an integrative analysis.
To reveal the common and specific patterns simultaneously, Zhang and Zhang \cite{Zhang2019a} proposed common specific matrix factorization (CSMF) by decomposing the data matrices into common and specific parts.
However, few of the existing methods consider the heterogeneity of the noise, and thus the data view of high noise may affect the analysis of that of low noise. 
To address this problem, Zhang and Zhang \cite{Zhang2019} employed the Bayesian approach and extended GLAD to Bayesian joint matrix decomposition (BJMD). Their experiments showed that considering the heterogeneous noise leads to superior performance in clustering. But theoretical analysis is still lacking.  

\par The matrix decomposition methods mentioned above have little relevance to the underlying computational architecture. They assumed that the program is running on a single machine, and an arbitrary number of data points are accessible instantaneously. However, the huge size of data often makes it impossible to handle all of them on a single machine. Many applications collect data distributedly from different sources (e.g., labs, hospitals). The communication between them is expensive due to the limited bandwidth, and direct data sharing also raise privacy concern.
Moreover, data collected from different sources often bear strong heterogeneous noise. Therefore, developing efficient matrix decomposition methods in a distributed system is essential. The commonly used computation architecture is that the overall data $\boldsymbol{X} \in R^{m\times n}$ is distributed onto $C$ node machines that are connected to a central processor. The desired methods should scale up well to distributed big data, communicate efficiently, and adequately tackle the heterogeneity of the noise. Researchers have developed many distributed matrix decomposition methods including distributed k-means \cite{Jin2006, Bahmani2012}, distributed BPMF \cite{Yu2014, Ahn2015, Qin2019}, distributed NMF \cite{Liu2010, Benson2014, Zdunek2017}, and so on. Developing efficient distributed algorithms for matrix decomposition methods should account for the partition strategy of the distributed data and then adopt an appropriate optimization strategy.
For example, when the number of instances $n$ is vast, and the number of features is small or moderate, i.e., the transposition of the data matrix $\boldsymbol{X}$ is tall-and-skinny,
it is usually convenient to split $\boldsymbol{X}$ over columns. We should then consider which optimization strategy is suitable for current partitioned data. However, few studies explored different optimization strategies and elaborated their difference. Even more serious is that few methods tackle the heterogeneity of noise among the distributed data.

To this end, we propose a distributed Bayesian matrix decomposition model (DBMD) that extends the BJMD for big data clustering and mining. We limit our scope to the data matrix whose transposition is tall-and-skinny, distribute it by columns onto node machines, and then focus on the optimization strategies for solving DBMD. Specifically, we adopt three strategies to implement the distributed computing, including 1) the accelerated descent gradient (AGD), 2) the alternating direction method of multipliers (ADMM), and 3) the communication-efficient accurate statistical estimation (CEASE), and then investigate their convergence behavior in the distributed setting. To tackle the heterogeneous noise, we propose an optimal plug-in weighted average that minimizes the variance of the estimation.
Extensive experiments verify our theoretical results, and the real-world experiments show the scalability and the effectiveness of our methods.

\par The contributions of this paper are as follows:
1) We propose a scalable distributed Bayesian matrix decomposition model for one big data matrix whose transposition is tall-and-skinny;
2) We adopt three optimization strategies and elaborate their differences in both empirical and theoretical perspectives;
3) We propose a flexible weighted average to tackle the heterogeneous noise and provide the theoretical result that is lacked in \cite{Zhang2019}.

\section{Preliminaries and Notations}
Throughout this paper, we use three standard mathematical notations including lightface lowercase ($x$), boldface lowercase ($\boldsymbol{x}$), boldface uppercase ($\boldsymbol{X}$) characters to represent scalars, vectors and matrices, respectively. $\boldsymbol{x}_{i \cdot}$, $\boldsymbol{x}_{\cdot j}$, $x_{ij}$ represent the $i$-th row, the $j$-th column, and the entry of the $i$-th row and the $j$-th of the matrix $\boldsymbol{X}$, respectively. Given a sequence of matrices $\{\boldsymbol{X}_c\}_{c=1}^C$, we use the following notations: $(\boldsymbol{x}_{i \cdot})_c$, $(\boldsymbol{x}_{\cdot j})_c$, $(x_{ij})_c$ to denote the corresponding row, column and entry in the $c$-th matrix $\boldsymbol{X}_c$. $\bar{\boldsymbol{X}}_c = \sum_{c=1}^C \boldsymbol{X}_c / C$ is the average of the matrix sequence $\{\boldsymbol{X}_c\}_{c=1}^C$.  $c \in [C]$ indicates that $c\in \{1, 2, \dots, C \}$.

\par Suppose a matrix variate function $f:R^{m\times n} \rightarrow R$ is convex and differentiable, and let $\nabla f$ denote the gradient of $f$. We say that $f$ is strongly convex with parameter $\mu_f > 0$, if for all $\boldsymbol{X}$, $\boldsymbol{Y} \in R^{m\times n}$
\begin{equation}
 f(\boldsymbol{X}) \geq f(\boldsymbol{Y}) + \langle \boldsymbol{Y}, \boldsymbol{X}- \boldsymbol{Y} \rangle + \frac{\mu_f}{2} \norm{\boldsymbol{X}-\boldsymbol{Y}}_F^2.
\end{equation}
When $\nabla f$ is Lipschitz continuous with parameter $L_f$, we then have
\begin{equation}
\norm{\nabla f(\boldsymbol{X}) - \nabla f(\boldsymbol{Y})}_F \leq L_f \norm{f(\boldsymbol{X}) - f(\boldsymbol{X})}_F.
\end{equation}
We denote the ratio of $L_f$ to $\mu_f$ as $\kappa_f = L_f / \mu_f$ if it exists.

\section{Related Work}
\subsection{Bayesian Matrix Decomposition}
Due to the flexibility and the effectiveness of Bayesian matrix decomposition, there have been a number of studies since PPCA was developed in 1999 \cite{Tipping1999}. At the same year, Bishop proposed the Bayesian PCA \cite{bishop1999bayesian}, which can automatically determine the number of retained principal components. To account for the complex type of noise in the real-world, PCA with exponential family noise \cite{Collins2001, Li2013} and its Bayesian variants have also been proposed \cite{mohamed2009bayesian}.

\par Similar to the Bayesian PCA, Bayesian k-means has also been proposed and it can automatically determine the number of clusters \cite{Welling2006}. Salakhutdinov and Mnih proposed the Bayesian probabilistic matrix factorization (BPMF) for predicting user preference for movies \cite{Salakhutdinov2008}, which places the Gaussian priors over both the basis and coefficient matrices. Moreover, Saddiki \textit{et al}. proposed GLAD that utilized three typical distributions as priors \cite{Saddiki2015} in a more flexible manner.

\par Very recently, Zhang and Zhang proposed BJMD to tackle the heterogeneous noise of multi-view data \cite{Zhang2019}. Let's denote the multi-view data as $\boldsymbol{X}_c$, $c \in [C]$, where $\boldsymbol{X}_c \in R^{m\times n_c}$ indicates the data collected is from the $c$-th source. BJMD assumes that the observed data matrices $\boldsymbol{X}_c$ share the same basis matrix $\boldsymbol{W}$ and different coefficient matrices $\boldsymbol{H}_c$ and use the Gaussian distributions of different variances to model the heterogeneity of the noise. Similar to GLAD, BJMD puts the Laplace prior onto the basis matrix $\boldsymbol{W}$ to pursue sparsity and the Dirichlet prior onto the columns of $\boldsymbol{H}_c$ to enhance the interpretability. Moreover, two efficient algorithms via variational inference and maximum a posterior respectively have been developed for solving it. They are much faster than GLAD, and thus are applicable to relatively large data. But BJMD is still a single machine methodology.

\subsection{Distributed Matrix Decomposition}
\par  
It is known that a proper initialization for the k-means algorithm is crucial for obtaining a good final solution. But the typical single machine initialization algorithms such as k-means++ \cite{Arthur2007} are sequential, which limits its applicability to big data. Generally, scaling the k-means algorithm to distributed data is relatively easy due to its iterative nature. Distributed k-means algorithms often split data by samples. The distributed versions of k-means often focus on reducing the number of passes needed to obtain a good initialization by sampling, e.g., DKEM \cite{Jin2006} and scalable k-means++ \cite{Bahmani2012}.

\par Recently, Yu \textit{et al}. proposed a distributed BPMF by splitting the data by samples and employed a stochastic alternating direction method of multipliers (ADMM) to solve it \cite{Yu2014}. But it is common in the filtering collaborative that the user-item data $\boldsymbol{X}\in R^{m\times n}$ are very spare and of large $m$ and $n$. Splitting $\boldsymbol{X}$ by  rows is only efficient for the tall-and-skinny matrix due to the communication load. To reduce the communication load, some studies split the data matrix $\boldsymbol{X}$ over both columns and rows, and then store the blocks of $\boldsymbol{X}$ distributedly on node machines \cite{Ahn2015, Qin2019}. Then they employed distributed Monte Carlo Markov Chain methods (MCMC) for inference. There also exist distributed NMF variants that split $\boldsymbol{X}$ into blocks \cite{Liu2010, Zdunek2017}.
Moreover, Benson \textit{et al}. proposed an approximated and scalable NMF algorithm for tall-and-skinny matrices.
Inspired by Donoho and Stodden \cite{Donoho2004}, they assume that the matrix $\boldsymbol{X}$ is nearly separable, i.e.
\begin{equation}
    \boldsymbol{X} = \boldsymbol{X}(:, \mathcal{K})\boldsymbol{H} + \boldsymbol{E},
\end{equation}
where $\mathcal{K}$ is an index set with size $r$, $\boldsymbol{X}(:, \mathcal{K})$ is the submatrix of $\boldsymbol{X}$ restricted to the columns indexed by $\mathcal{K}$, and $E$ is a noise matrix.
This algorithm only requires one round iteration \cite{Benson2014}.

\section{DBMD}
\subsection{Model Construction}
Suppose that the data matrix $\boldsymbol{X} \in R^{m \times n}$, where $m$ is the number of features, $n$ is the number of samples and $n \gg m$. Since $n$ is very large, $\boldsymbol{X}$ cannot be handled by a single machine. Therefore, $\boldsymbol{X}$ is split by columns and  distributedly stored on $C$ machines: $\{\boldsymbol{X}_c\}_{c=1}^C$, where $\boldsymbol{X}_c \in R^{m \times n_c}$ and $\sum_{c=1}^C n_c = n$. Inspired by a recent study \cite{Zhang2019}, we assume that $\{\boldsymbol{X}_c\}_{c=1}^C$  share the same basis matrix $\boldsymbol{W}$ and are generated as follows
\begin{equation}\label{eq:gneration}
\boldsymbol{ X}_c = \boldsymbol{W}\boldsymbol{H}_c + \boldsymbol{E}_c,
\end{equation}
where $\boldsymbol{W} \in R^{m \times r}$, $\boldsymbol{H}_c \in R^{r \times n_c}$ and $\boldsymbol{E}_c$  is the IID Gaussian noise, $(e_{ij})_c \sim N(0, \sigma_c^2)$. We further put a zero-mean Laplace prior on $\boldsymbol{W}$ to enforce its sparsity
\begin{equation}
w_{ik} \sim p(w_{ik}| 0, \lambda_0),
\end{equation}
where the density function
\begin{equation}
p(y|\mu, \lambda) = \frac{1}{2\lambda} \exp\left( - \frac{|y - \mu|}{\lambda}\right).
\end{equation}
We put a Dirichlet prior $\text{Dir}(\boldsymbol{\alpha}_0)$ on each column $(h_{\cdot j})_c$ of $\boldsymbol{H}_c$
\begin{equation}
(\boldsymbol{h_{\cdot j}})_c \sim p((\boldsymbol{h_{\cdot j}})_c| \boldsymbol{\alpha_0}),
\end{equation}
where $\boldsymbol{\alpha}_0 > 0$ is a $r$-dimensional vector and the density function of $\text{Dir}(\boldsymbol{\alpha}_0)$ is
\begin{equation}
p(\boldsymbol{y}|\boldsymbol{\alpha}_0) = \frac{\Gamma(\sum_{i=1}^r \alpha_{i})}{\prod_{i=1}^r\Gamma(\alpha_{0i})}\prod_{i=1}^r y_i^{\alpha_{i} - 1}.
\end{equation}
The support of the Dirichlet is $y_1, \dots, y_r$, where $y_i \in (0, 1)$ and $\sum_{i=1}^{r} y_i = 1$, which is a unit simplex. Note that the Dirichlet prior restricts the $\boldsymbol{H}_c$ to be non-negative and the columns sum of $\boldsymbol{H}_c$  all equal one. Therefore, $(h_{\cdot j})_c$ can be interpreted as a vector indicates the membership of clusters. The Gaussian noise $\epsilon_c$ models the noise level in the corresponding $\boldsymbol{X}_c$
\begin{equation}
\epsilon_c \sim p(\epsilon_c|0, \sigma_c^2),
\end{equation}
where
\begin{equation}
p(y|\mu, \sigma_c^2) = \frac{1}{\sqrt{2\pi}\sigma_c}\exp \left(-\frac{(y-\mu)^2}{2\sigma_c^2} \right).
\end{equation}

\par We are interested in the posterior of $\boldsymbol{W}$ and $\boldsymbol{H}_c$. By the Bayes' theorem, it is proportional to the complete likelihood, which can be written as
\begin{flalign}
\begin{split} \label{eq:nll}
& p(\boldsymbol{W}, \boldsymbol{H}_1, \ldots, \boldsymbol{H}_C, \sigma_1^2, \ldots, \sigma_C^2, \boldsymbol{X}_1, \ldots,\boldsymbol{X}_C; \lambda_0, \boldsymbol{\alpha_0})  \\
= & p(\boldsymbol{W};\lambda)\prod_{c=1}^C p(\boldsymbol{X}_c|\boldsymbol{W}, \boldsymbol{H}_c, \sigma_c^2)p(\boldsymbol{H}_c;\boldsymbol{\alpha_0}). \\
\end{split}
\end{flalign}
Maximum a posterior can then be  formulated as minimizing the negative log likelihood
\begin{flalign}
\begin{split}\label{eq:opt1}
\min_{\boldsymbol{W}, \boldsymbol{H}_c} & \sum_{c=1}^C \frac{1}{2 \sigma_c^2}||\boldsymbol{X}_c - \boldsymbol{W}\boldsymbol{H}_c||_F^2  + \frac{1}{\lambda_0} ||\boldsymbol{W}||_1\\
     &  - \sum_{c=1}^C \sum_{k,j} (\alpha_{0k} -1) \ln (h_{kj})_c + \sum_{c=1}^C mn_c \ln \sigma_c \\
\st  & \sum_{k=1}^r (h_{kj})_c = 1, (h_{kj})_c > 0,
\end{split}
\end{flalign}
where the first term  is essentially the weighted sum of the goodness of the approximations across the $C$ node machines.
Intuitively, the weight $1/2\sigma_c^2$ gives higher importance to the clean data ($\sigma_c^2$ is small) and lower importance to the noisy data. The second term is the $L_1$-norm regularizer, which enforces the sparsity of the basis matrix $\boldsymbol{W}$. The third term is due to the Dirichlet prior and regularizes the coefficient matrices $\{\boldsymbol{H}_c\}_{c=1}^C$.
Specifically, the third term is minimized when  $(h_{kj})_c=\alpha_{0k}/\sum_{k=1}^{r} \alpha_{0k}$, which is the expectation of the Dirichlet prior. Therefore, the third term enforces $(\boldsymbol{h_{\cdot j}})_c$ to the prior and reduces the risk of overfitting.

For the ease of presentation, we denote  $\boldsymbol{\alpha} =  \boldsymbol{\alpha_0} - \boldsymbol{1}$, $\lambda = 1 / \lambda_0$ and set $\sigma_c=1$ (we will discuss the situation where $\sigma_c$ is inferred later ). Let's rewrite the optimization problem in Eq. (\ref{eq:opt1}) as follows
\begin{flalign}
\begin{split}\label{eq:opt2}
\min_{\boldsymbol{W}, \boldsymbol{H}_c} & \sum_{c=1}^C \frac{1}{2}||\boldsymbol{X}_c - \boldsymbol{W}\boldsymbol{H}_c||_F^2 + \lambda ||\boldsymbol{W}||_1\\
&  - \sum_{c=1}^C \sum_{k,j} \alpha_k \ln (h_{kj})_c\\
\st  & \sum_{k=1}^r (h_{kj})_c = 1, (h_{kj})_c > 0.
\end{split}
\end{flalign}
Note that Eq. (\ref{eq:opt2}) is bi-convex to $\boldsymbol{W}$ and $\boldsymbol{H}_c$. A common approach is to update $\boldsymbol{W}$ and $\boldsymbol{H}_c$ alternatively.

\subsection{Accelerated Gradient Decent Optimization}
Recall that $\{\boldsymbol{X}_c\}_{c=1}^C$ is distributedly stored on $C$ node machines. We store the corresponding $\boldsymbol{H}_c$ on the $c$-th node machines and store $\boldsymbol{W}$ on the central processor.
Then we can easily adopt the maximum a posterior algorithm in \cite{Zhang2019} to solve Eq. (\ref{eq:opt2}) by updating $\boldsymbol{W}$ and $\{ \boldsymbol{H}_c \}_{c=1}^C$  alternatively. Specifically, we update  $\boldsymbol{W}$ with other parameters fixed
\begin{equation}\label{eq:vanilla_update_W}
\min_{\boldsymbol{W}} \frac{1}{2}||\boldsymbol{X}_c - \boldsymbol{W}\boldsymbol{H}_c||_F^2  + \lambda ||\boldsymbol{W}||_1.\\
\end{equation}
Let's denote $f(\boldsymbol{W}) = \sum_{c=1}^C f_c(\boldsymbol{W})$, $f_c(\boldsymbol{W}) =  \frac{1}{2} \norm{\boldsymbol{X}_c -  \boldsymbol{W}\boldsymbol{H}_c}_F^2$ and $g(\boldsymbol{W}) = \lambda \norm{\boldsymbol{W}}_1$.
The objective function Eq. (\ref{eq:vanilla_update_W}) consists of the non-smooth $L_1$-norm regularizer $g(\boldsymbol{W})$ and the quadratic loss terms.
Therefore, it can be efficiently solved by the fast iterative shrinkage-thresholding algorithm (FISTA) \cite{Beck2009}.
FISTA is an  accelerated gradient decent (AGD) algorithm and enjoys a quadratic convergence rate.
Specifically, we construct two sequences $\{\boldsymbol{Y}^k\}$ and $\{\boldsymbol{W}^k\}$, and alternatively update them
\begin{equation}\label{eq:argmin_wc}
\boldsymbol{W}^k = \argmin_{ \boldsymbol{W} } g(\boldsymbol{W}) + \frac{L}{2}  \norm{ \boldsymbol{W} -  \left( \boldsymbol{Y}^k - \frac{1}{L} \nabla \tilde{f}_c(\boldsymbol{Y}^k) \right)}_F^2
\end{equation}
and
\begin{equation}
\boldsymbol{Y}^{k+1} = \boldsymbol{W}^k + \frac{\nu_k -1}{\nu_{k+1}} (\boldsymbol{W}^k - \boldsymbol{W}^{k - 1}),
\end{equation}
where $L_f =\norm{ \sum_{c=1}^C H_cH_c^T}_2 > 0$ is the Lipschitz constant of $\sum_c \nabla f_c(\boldsymbol{W})$.
$\boldsymbol{W}^k$ contains the approximate solution by minimizing the proximal function.
Eq. (\ref{eq:argmin_wc}) has the closed-form solution
$$\boldsymbol{W}^k = \mathcal{S}_{\lambda/L}\left(\boldsymbol{Y}^k - \frac{1}{L} \sum_c \nabla  \boldsymbol{f}_c(\boldsymbol{Y}^k)\right).$$
where
$$
\mathcal{S}_{\lambda/L}(\boldsymbol{X})= \text{sign}(\boldsymbol{X}) \circ \max(|\boldsymbol{X}| - \lambda/L, 0),
$$
is the soft thresholding operator, and $\circ$ is the Hadamard product.
$\boldsymbol{Y}^{k+1}$ stores the search point constructed by linearly combining the latest two approximate solutions, i.e., $\boldsymbol{W}^{k-1}$ and $\boldsymbol{W}^{k}$. The combination coefficient $\nu_{t+1}$ was carefully designed in \cite{Beck2009} as follows
\begin{equation}
\nu_{k+1} = \frac{1 + \sqrt{4\nu_k^2 + 1}}{2}.
\end{equation}

\par At each iteration, we  broadcast the current $\boldsymbol{Y}^k$ to all node machines and compute the gradients. Then we collect the computed gradients to the central processor and update $\boldsymbol{W}^k$. We iteratively update $\boldsymbol{W}^k$ and $\boldsymbol{Y}^k$ until it converges.
\begin{algorithm}[!t]
    \caption{Updating $\boldsymbol{W}$ with AGD}
    \begin{algorithmic}[1]
        \Input Initial $\boldsymbol{W}^0$, $\boldsymbol{Y}^0 = \boldsymbol{W}^0$, $k = 0$, $\nu_0 = 1$
        \Repeat
        \State Each node machine computes $\nabla f_c(\boldsymbol{Y^k})$ and sends to the central processor
        \State The central processor computes
        $$\boldsymbol{W}^k = \mathcal{S}_{\lambda/L}\left(\boldsymbol{Y}^k - \frac{1}{L} \sum_c \nabla \boldsymbol{f}_c(\boldsymbol{Y}^k)\right).$$
        \State The central processor computes $\nu_{k+1} = \frac{1 + \sqrt{4\nu_k^2 + 1}}{2}$
        \State The central processor computes
        $$\boldsymbol{Y}^{k+1} = \boldsymbol{W}_c^k + \frac{\nu_k -1}{\nu_{k+1}} (\boldsymbol{W}_c^k - \boldsymbol{W}_c^{k - 1})$$
        and broadcasts to the node machines
        \State $k \leftarrow k + 1$
        \Until{Convergence.}
    \end{algorithmic}
\label{alg: agd}
\end{algorithm}

\par Updating $\{ \boldsymbol{H}_c \}_{c=1}^C$  is straightforward. We broadcast $\boldsymbol{W}$  from the central processor to the node machines, and then we can adopt the same algorithm in \cite{Zhang2019} to solve the following problem
\begin{equation}
\begin{split}\label{eq:navive_update_Hc}
\min_{\boldsymbol{H}_c} &  \frac{1}{2}||\boldsymbol{X}_c - \boldsymbol{W}\boldsymbol{H}_c||_F^2 - \sum_{c=1}^C \sum_{k,j} \alpha_k \ln (h_{kj})_c\\
\st  & \sum_{k=1}^r (h_{kj})_c = 1, (h_{kj})_c > 0,
\end{split}
\end{equation}
with respect to $\boldsymbol{H}_c$ in parallel.

\par Note that solving Eq. (\ref{eq:navive_update_Hc}) requires no data communication. Therefore, it has no difference from the single machine algorithm. However, one concern arises when we  update $\boldsymbol{W}$ in a distributed system, broadcasting $\boldsymbol{W}$ to the node machines and collecting the gradients from the node machines can be very time-consuming. The speed of inter-node communication can be much slower than that of intra-node computation in the distributed system \cite{Lan2018}. Therefore, the data communication is often the bottleneck of the distributed algorithm and  updating $\boldsymbol{W}$ requires carefully consideration. The algorithm of updating $\boldsymbol{W}$ with AGD is summarized in Fig. 1A and Algorithm \ref{alg: agd}. In the following, we discuss how to design two efficient algorithms from optimization  and statistical perspectives, respectively.
\begin{figure*}
  \centering
  \includegraphics[width=\linewidth]{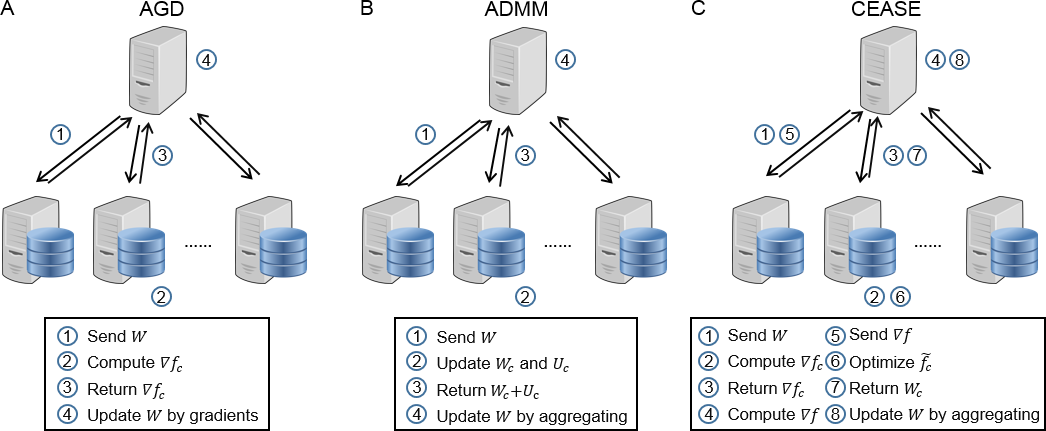}
  \caption{Illustration of updating $\boldsymbol{W}$. AGD optimizes $\boldsymbol{W}$ on the central processor.
      ADMM and CEASE update $\boldsymbol{W}_c$ on the node machines and aggregate them on the central processor.
      The steps at each round of iteration are numbered sequentially.}
\end{figure*}

\subsection{Efficient Distributed Optimization}
We adopt the ADMM to implement the algorithm of updating $\boldsymbol{W}$. Note that $\boldsymbol{W}$ in Eq. (\ref{eq:vanilla_update_W}) is a global variable. We formulate the following optimization problem
\begin{equation}
\begin{split}\label{eq:admm_update_W}
\min_{\boldsymbol{W}, \boldsymbol{W}_c } & \sum_{c=1}^C \frac{1}{2}||\boldsymbol{X}_c - \boldsymbol{W}_c\boldsymbol{H}_c||_F^2  + \lambda ||\boldsymbol{W}||_1\\
\st  & \boldsymbol{W}_c = \boldsymbol{W},
\end{split}
\end{equation}
where $\boldsymbol{W}$ is the consensus variable.
One can easily verify that Eq. (\ref{eq:admm_update_W}) is equivalent to Eq. (\ref{eq:vanilla_update_W}).
Let's write down its augmented Lagrangian
\begin{align}\label{eq:admm_aug_L}
\begin{split}
L_\rho(\boldsymbol{W}, \boldsymbol{W}_c, \boldsymbol{U}_c)  =   \sum_{c=1}^C \frac{1}{2}\norm{\boldsymbol{X}_c -\boldsymbol{W}_c\boldsymbol{H}_c}_F^2 + \lambda \norm{\boldsymbol{W}}_1 \\
+ \sum_{c=1}^{C} \langle \boldsymbol{U}_{c}, \boldsymbol{W} - \boldsymbol{W}_c \rangle  + \frac{1}{2} \rho \sum_{c=1}^{C} \norm{\boldsymbol{W} - \boldsymbol{W}_c}_F^2 ,
\end{split}
\end{align}
where $\rho > 0$ is the penalty parameter, and $\boldsymbol{U}_c$ is the corresponding dual variables.
$\boldsymbol{W}$ is the global variable stored on the central processor, and $\boldsymbol{W}_c$, $\boldsymbol{U}_c$, $\boldsymbol{H}_c$ are locally stored on the node machines. The ADMM at the ($k+1$)-th iteration consists of the following steps
\begin{flalign}
\boldsymbol{W}_c^{k+1} =     & \argmin_{\boldsymbol{W}_c} L_\rho(\boldsymbol{W}^k, \boldsymbol{W}_c, \boldsymbol{U}_c^k),   \label{eq:admm_W_c_iter} \\
\boldsymbol{W}^{k+1}      =      & \argmin_{\boldsymbol{W}} L_\rho(\boldsymbol{W}, \boldsymbol{W}_c^{k + 1}, \boldsymbol{U}_c^k),   \label{eq:admm_W_*_iter} \\
\boldsymbol{U}_{c}^{k+1} = &  \boldsymbol{U}_{c}^k + \rho (\boldsymbol{W}^{k+1} - \boldsymbol{W}_c^{k+1}).
\end{flalign}
Both Eq. (\ref{eq:admm_W_c_iter}) and Eq. (\ref{eq:admm_W_*_iter}) have the closed-form solutions
\begin{flalign}
\boldsymbol{W}_c^{k+1} =     & \left( \frac{\boldsymbol{X}_c\boldsymbol{H}_c^T + \boldsymbol{U}_c^k}{\rho} + \boldsymbol{W}^k \right)\left(\boldsymbol{I}_r + \frac{\boldsymbol{H}_c\boldsymbol{H}_c^T}{\rho}\right)^{-1},   \label{eq:admm_W_c_sol}\\
\boldsymbol{W}^{k+1}      =      & \mathcal{S}_{\lambda / C\rho}(\bar{\boldsymbol{W}}_c^{k+1} - \bar{\boldsymbol{U}}^k_c /\rho) \label{eq:admm_W_*_sol},
\end{flalign}
\begin{algorithm}[!t]
    \caption{Updating $\boldsymbol{W}$ with ADMM}
    \begin{algorithmic}[1]
        \Input initial  $\boldsymbol{W}^0$, $\boldsymbol{W}_c^0$, $\boldsymbol{U}_c^0$, $\rho$, $k = 0$
        \Repeat
        \State Computes $\boldsymbol{W}_c^k$ by Eq. (\ref{eq:admm_W_c_sol}) on each node machine
        \State Computes $\boldsymbol{W}_c - \boldsymbol{U}_c^k / \rho$ in each node machine and sends to the central processor
        \State The central processor obtains $\boldsymbol{W}^{k+1}$  by aggregating
        \[
        \boldsymbol{W}^{k+1}=\mathcal{S}_{\lambda / C\rho}(\bar{\boldsymbol{W}}_c^{k+1} - \bar{\boldsymbol{U}}^k_c /\rho)
        \]
        \State $k \leftarrow k + 1$
        \Until{Convergence}
    \end{algorithmic}
\label{alg: admm}
\end{algorithm}
where $\mathcal{S}_{\lambda / C\rho}$ is the soft thresholding operator with parameter $\lambda / C\rho$. Now, $\boldsymbol{W}_c$ is optimized locally on the node machines in parallel and requires no data communication. $\boldsymbol{U}_c$ is also locally optimized in parallel. The only step that involving data communication is updating the global variable $\boldsymbol{W}$. It is simply $\bar{\boldsymbol{W}}_c^{k+1} - \bar{\boldsymbol{U}}_c^k /\rho$ and taking soft thresholding operation. We compute $\boldsymbol{W}_c^k -  \boldsymbol{U}_c^k$ on each node machine in parallel and then aggregate the results on the central node. We then apply the thresholding operator to obtain the new $\boldsymbol{W}$ and broadcast  it to all node machines. Note that at each iteration, we only need to collect and broadcast a matrix of size $m\times r$. Therefore, the data communication load has been significantly reduced.
The algorithm of updating $\boldsymbol{W}$ with ADMM is summarized in Fig. 1B and Algorithm \ref{alg: admm}.

\subsection{Efficient Distributed Statistical Inference}
We can also solve Eq. (\ref{eq:vanilla_update_W}) from a statistical perspective. Recent advances on distributed statistical inference \cite{Jordan2019, fan2019communication} provide us with powerful tools.
Here we use the CEASE to develop an efficient distributed statistical procedure due to its effectiveness.
Let $\boldsymbol{W}$ at the $k$-th iteration be $\boldsymbol{W}^k$. Following the scheme of CEASE, each node machine computes
\begin{align}\label{eq:cease_update_W_c}
\begin{split}
\boldsymbol{W}_c^k = \argmin_{\boldsymbol{W}} \tilde{f}_c(\boldsymbol{W}),
\end{split}
\end{align}
where
\begin{align}\label{eq:cease_loss}
\begin{split}
\tilde{f}_c(\boldsymbol{W}) = f_c(\boldsymbol{W}) - \langle \nabla f_c(\boldsymbol{W}^k) - \nabla f(\boldsymbol{W}^k), \boldsymbol{W} \rangle
\\+ \frac{\gamma}{2} \norm{\boldsymbol{W} - \boldsymbol{W}^k}_F^2 + g(\boldsymbol{W}),
\end{split}
\end{align}
\begin{algorithm}[tbh]
    \caption{Updating $\boldsymbol{W}$ with CEASE}
    \begin{algorithmic}[1]
        \Input initial  $\boldsymbol{W}^0$, $\boldsymbol{W}_c^0$,  $\gamma$, $k = 0$
        \Repeat
        \State Computes $\nabla f_k(\boldsymbol{W}^k)$ on each node machine and sends to the central processor
        \State The central processor computes $\nabla f(\boldsymbol{W}^k) = 1 / C \sum_{c=1}^C \nabla f_k(\boldsymbol{W}^k)$
        and broadcasts to node machines. \label{algo:cease:grad}
        \State Computes $\boldsymbol{W}_c^k$ on each node machine by FISTA and sends to the central processor
        \State The central processor aggregates $\boldsymbol{W}^{k+1}  = \bar{\boldsymbol{W}}_c^k$  \label{algo:cease:wc}

        \State $k \leftarrow k + 1$
        \Until{Convergence}
    \end{algorithmic}
\label{alg: cease}
\end{algorithm}
where $\gamma\geq 0$ is the parameter of the proximal point algorithm. It is notably that the $f_c(\boldsymbol{W}) - \langle \nabla f_c(\boldsymbol{W}^k) - \nabla f(\boldsymbol{W}^k), \boldsymbol{W} \rangle $ is referred as the gradient-enhanced loss (GEL) function, in which the loss of the local data $\boldsymbol{X}_c$ is enhanced by the global gradient $\nabla f(\boldsymbol{W}^k)$. Conceptually, the global gradient adaptively enhances the similarity of the $f_c$ and thus accelerates the convergence. This idea of using GEL has also been explored in \cite{Shamir2014, Jordan2019}. When $\lambda = 0$,  there exists a closed-form solution. While $\lambda > 0$, Eq. (\ref{eq:cease_update_W_c}) consists of the non-smooth $L_1$-norm regularizer $g(\boldsymbol{W})$ and the remaining smooth terms that are merely sums of the quadratic loss term and the linear terms. It can also be efficiently solved by FISTA. The optimizing of $\boldsymbol{W}_c^k$ requires no data communication. The central processor collects  $\boldsymbol{W}_c^k$ and aggregates by taking average $\boldsymbol{W}^{k+1} = \frac{1}{C}\sum_{c=1}^C \boldsymbol{W}_c^k$. The whole algorithm of updating $W$ with CEASE is summarized in Fig. 1C and Algorithm \ref{alg: cease}. The data communication load of CEASE is twice of the ADMM due to  additionally broadcasting and sending $\nabla f (\boldsymbol{W}^k)$.

\subsection{Tackle the Heterogeneous Noise}
\par ADMM and CEASE do not account for the heterogeneity of the noise. Fortunately, they can be easily extended by plugging in the weighted average to achieve that. Given $\boldsymbol{W}_c$ and $\boldsymbol{H}_c$, the variance of the noise of each $\boldsymbol{X}_c$ is computed by $\sigma_c^2 = \norm{\boldsymbol{X}_c - \boldsymbol{W}_c \boldsymbol{H}_c}_F^2 / mn_c$, which can be derived by the maximum likelihood as showed in \cite{Zhang2019}. Then $\boldsymbol{H}_c$ can be separately updated on each node machine, and thus different noise levels will not influence the results. However, aggregating $\boldsymbol{W}_c$  corresponding to  different levels of noise by taking average is problematic. Let's consider the ADMM algorithm. Intuitively, $\boldsymbol{W}_c$ inferred from $\boldsymbol{X}_c$ of lower noise is more believable. Inspired by \cite{Zhang2019}, we adopt the weighted average to aggregate $\boldsymbol{W}_c$
\begin{equation}
\tilde{\boldsymbol{W}_c} =\sum_{c=1}^{C} \frac{1/ \sigma_c^2}{\sum_{c=1}^C 1/\sigma_c^2} \boldsymbol{W}_c.
\end{equation}
Note that the weight of $\boldsymbol{W}_c$ with small variance $\sigma_c^2$ is higher. CEASE can also be easily modified by plugging in the weighted average. Note that CEASE takes the average of both the gradients  $\nabla f_k(\boldsymbol{W}^k)$ (Algorithm \ref{alg: cease}, line~\ref{algo:cease:grad}) and  $\boldsymbol{W}_c$ (Algorithm \ref{alg: cease}, line~\ref{algo:cease:wc}) on the central processor. Thus, we can also use the weighted average versions to aggregate the gradients and $\boldsymbol{W}_c$, respectively.

\subsection{Computational Remarks on Updating $\boldsymbol{W}$}

\subsubsection{The Optimality of the Weighted Average}
\begin{assumption}\label{assumption:matrix normal}
    $\boldsymbol{X}_c$ follows the matrix normal distribution with isotropic covariances, $\boldsymbol{X}_c \sim \mathcal{MN}(\boldsymbol{X}_c^*, \sigma_c I, \sigma_c I)$.
\end{assumption}
\begin{assumption}\label{assumption:Hc}
    $\{\boldsymbol{H}_c\}_{c=1}^C$ are of the same size and each column of $H_c$ follows the same distribution and the expectation $E(\boldsymbol{H}_c \boldsymbol{H}_c^T)$ exists.
\end{assumption}
Note that the Assumption \ref{assumption:matrix normal} is equivalent to the generation process Eq. (\ref{eq:gneration}). The estimated $\tilde{\boldsymbol{W}}$ and $\bar{\boldsymbol{W}}$ are also random matrices. We use the sum of the entry-wise variances to measure the variances, e.g., $\var(\tilde{\boldsymbol{W}}) = \sum_{i, k} \var(\tilde{w}_{ik})$.
For convenience, we assume that $\{\boldsymbol{H}_c\}_{c=1}^C$ are of the same size.
\begin{theorem}\label{thn:variance reduction}
Let Assumptions \ref{assumption:matrix normal} and \ref{assumption:Hc} hold and let $\lambda = 0$ and $k \to \infty$. The variance ratio
\begin{equation}
\frac{\var(\tilde{\boldsymbol{W}})}{\var(\bar{\boldsymbol{W}}) } = \frac{\sum_c C/(1/ \sigma_c^2)}{\sum_c \sigma_c^2 / C}\leq 1
\end{equation}
The equality reaches if and only if all $\sigma_c^2$ are equal. Moreover, the $\tilde{\boldsymbol{W}}$ is the optimal weighted average that minimizes the variance.
\end{theorem}
\begin{proof}[Proof]
    We first consider the ADMM algorithm.  
    Both $f$ and $g$ are closed, proper and convex. A previous study \cite[Section 3.2.1]{boyd2011distributed} has shown that the dual variable $\boldsymbol{U}_c$ converges to $\boldsymbol{U}_c^*$ with  $k \to \infty $. So we treat $\boldsymbol{U}_c^*$ as a constant matrix. Note that
    \begin{equation}
    \boldsymbol{X}_c \boldsymbol{H}_c^T / \rho \sim \mathcal{MN}(\boldsymbol{X}_c^*\boldsymbol{H}_c^T/\rho, \sigma_cI/\rho, \sigma_c \boldsymbol{H}_c^T\boldsymbol{H}_c/\rho).
    \end{equation}
    Based on the Eq. (\ref{eq:admm_W_c_iter}) and the property of the matrix normal distribution, we have
    \begin{equation}
    \var(\boldsymbol{W}_c) =\sigma_c^2 \tr(I \otimes \boldsymbol{\Lambda}_c \boldsymbol{H}_c^T \boldsymbol{H}_c \boldsymbol{\Lambda}_c),
    \end{equation}
    where $\Lambda_c = (\boldsymbol{I} + \boldsymbol{H}_c \boldsymbol{H}_c^T)^{-1}$. The variance ratio
    \begin{align}
    \begin{split}
    \frac{\var(\tilde{\boldsymbol{W}})}{\var(\bar{\boldsymbol{W}}) } &=\frac{\sum_c \frac{1/\sigma_c^4}{(\sum_c 1/ \sigma_c^2)^2}\var(\boldsymbol{W}_c)}{\sum_c \frac{1}{C^2}\var(\boldsymbol{W}_c)}  \\
    &= \frac{\sum_c \frac{1 / \sigma_c^2}{(\sum_c 1/ \sigma_c^2)^2}\tr(\boldsymbol{\Lambda}_c \boldsymbol{H}_c^T \boldsymbol{H}_c^T \boldsymbol{\Lambda}_c)}{\sum_c \frac{\sigma_c^2}{C^2} \tr(\boldsymbol{\Lambda}_c \boldsymbol{H}_c^T \boldsymbol{H}_c^T)} \\
    & = \frac{\sum_c C/(1/ \sigma_c^2)}{\sum_c \sigma_c^2 / C}  \leq 1
    \end{split}
    \end{align}
    Based on the inequality of arithmetic and geometric means, the ration is not greater than 1.
    \par
    Then we prove the optimality. Consider minimizing the variance of the weighted average $v_c\boldsymbol{W}_c$.
    \begin{equation}\label{eq:min_v}
    \min \sum_c v_c^2 \sigma_c^2 a,
    \end{equation}
    with $\boldsymbol{v} \geq 0, \sum_c v_c = 1$. Denote $a = \tr(\boldsymbol{\Lambda}_c \boldsymbol{H}_c^T \boldsymbol{H}_c^T)$. Eq. (\ref{eq:min_v}) is a constrained quadratic programming. The KKT condition
    \begin{equation}
    v_c \sigma_c^2 a - \mu = 0, c \in [C]
    \end{equation}
    where $\mu > 0$ is the dual variable of the Lagrangian. Then $v_i = \frac{1 / \sigma_c^2}{\sum_c 1 / \sigma_c^2}$ and $\mu = \frac{a}{\sum_c 1 / \sigma_c^2}$ is the solution of the system of equations.
\end{proof}
Note that the numerator is the harmonic average of $\sigma_c^2$ and the denominator is the arithmetic average of $\sigma_c^2$. Theorem \ref{thn:variance reduction} shows that the weighted average  $\tilde{\boldsymbol{W}}$ can  reduce the variance of $\bar{\boldsymbol{W}}$, and the weights are optimal. The result holds for both ADMM and CEASE algorithms. When $\{\boldsymbol{H}_c\}_{c=1}^C$ are of different size, there exists similar result that can be proved with the same procedure.

\subsubsection{Convergence Rate}
In this section, we discuss the convergence rate of updating $\boldsymbol{W}$. In particular, we concern about the effect of the number of instances increasing on the convergence rate.
\begin{lemma}\label{lm:strong-convex}
   $f(\boldsymbol{W})$ is strongly convex with parameter $\mu_f = \sigma_{\text{min}}(\sum_{c=1}^C \boldsymbol{H}_c\boldsymbol{H}_c^T)$, and
   $\nabla f(\boldsymbol{W})$  is Lipschitz continuous with parameter $L_f = \sigma_{\text{max}}(\sum_{c=1}^C \boldsymbol{H}_c\boldsymbol{H}_c^T)$, where
   $\sigma_{\text{min}}(\boldsymbol{A})$ and $\sigma_{\text{max}}(\boldsymbol{A})$ indicate the smallest and the largest eigenvalue of matrix $A$, respectively. The ration $\kappa_f = L_f / \mu_f$ exists.
\end{lemma}
\begin{proof}[Proof]
    $f(\boldsymbol{W})$ is twice differentiable and we have
    \begin{equation}
    \nabla f(\boldsymbol{W}) =\sum_{c=1}^C (\boldsymbol{W}\boldsymbol{H}_c - \boldsymbol{X}_c) \boldsymbol{H}_c^T, \nabla^2 f(\boldsymbol{W}) = \sum_{c=1}^C (\boldsymbol{H}_c \boldsymbol{H}_c^T) \otimes I
    \end{equation}
    Note that
    \begin{equation}
    \nabla^2 f(\boldsymbol{W})\succeq \sigma_{\text{min}}\left(\sum_{c=1}^C \boldsymbol{H}_c \boldsymbol{H}_c^T\right) I
    \end{equation}
    It implies that $f(\boldsymbol{W})$ is strongly convex with parameter $\mu_f = \sigma_{\text{min}}(\sum_{c=1}^C \boldsymbol{H}_c\boldsymbol{H}_c^T)$.
    For $\forall \boldsymbol{W}, \boldsymbol{Y} \in R^{m\times r}$, we have
    \begin{align}
        \begin{split}
    \norm{\nabla f(\boldsymbol{W}) - \nabla f(\boldsymbol{Y})}_F &= \norm{(\boldsymbol{W}-\boldsymbol{Y}) \sum_{c=1}^C\boldsymbol{H}_c \boldsymbol{H}_c^T}_F \\
    & \leq  L_f \norm{\boldsymbol{W} - \boldsymbol{Y}}_F,
    \end{split}
    \end{align}
    where $L_f = \sigma_{\text{max}}(\sum_{c=1}^C \boldsymbol{H}_c \boldsymbol{H}_c^T)$. $\nabla f(\boldsymbol{W})$  is Lipschitz continuous with parameter $L_f$.
\end{proof}

\begin{lemma}\label{lm:nbla}
    There exists $\delta > 0$, such that $\norm{\nabla^2 f(\boldsymbol{W}) - \nabla^2 f_c(\boldsymbol{W})}_F \leq \delta$ holds for all $c \in [C]$ and $\boldsymbol{W} \in R^{m\times r}$.
\end{lemma}
Lemma \ref{lm:strong-convex} and \ref{lm:nbla} characterize the smoothing part of the objective function. Suppose that the Assumption \ref{assumption:Hc} holds. It is easy to verify that $L_f$ and $\mu_f$ grow linearly with the increasing of the number of instances ($n_c$) in each node machine.
\begin{proof}[Proof]
    Note
    \begin{equation}
    \norm{\nabla^2 f(\boldsymbol{W}) - \nabla^2 f_c(\boldsymbol{W})}_2 =  \norm{\sum_{l\neq c}\boldsymbol{H}_l\boldsymbol{H}_l^T \otimes I}_2 = \delta_c.
    \end{equation}
    Thus, there exists and $\delta = \max \{\delta_c\}_{c=1}^C$.
\end{proof}


\par The FISTA algorithm is known to have a quadratic convergence rate \cite[Theorem 4.4]{Beck2009}. But the contraction factor was not given in this work. Tao \textit{et al} \cite[Thereom 5.5]{Tao2016} provided the contraction factor $\tau_1$ of the local convergence. The contraction factor of FISTA, $\tau_1$ depends on the structural parameter $\kappa_f$. The contraction factor of the CEASE algorithm was given in \cite[Theorem 3.1]{fan2019communication}.
\begin{theorem}\label{thn:cease-convergence}
    Consider $\{ \boldsymbol{W}^k \}$ generated by Algorithm \ref{alg: cease}. Suppose that $\delta^2 / (\mu_f + \gamma)^2< \mu_f / (\mu_f + 2\gamma)$. We have
    \begin{equation}
    \norm{\boldsymbol{W}^{k+1} - \boldsymbol{W}^*}_F \leq \norm{\boldsymbol{W}^{k} - \boldsymbol{W}^*}_F \tau_3,
    \end{equation}
    where $\boldsymbol{W}^*$ is the KKT point, $\tau_3 = \frac{\delta\sqrt{\mu_f^2 + 2\gamma \mu_f} + \gamma}{(\mu_f + \gamma)^2} < 1$ is the contraction factor.
\end{theorem}

\par Theorem \ref{thn:cease-convergence} implies that the CEASE algorithm converges faster with $n_c$ increasing.
It is intuitive that the estimation of $\boldsymbol{W}$ in the node machine, i.e., $\boldsymbol{W}_c^k$,  is more accurate when $n_c$ is sufficiently large. Therefore, it takes fewer rounds of aggregation before the algorithm converges. But it is not the case for the AGD algorithm, the structural parameter $\kappa_f$ remains stable with $n_c$ increasing (Fig. \ref{fig:hc_eigen}, right). So we do not expect Algorithm \ref{alg: agd} converges faster with $n_c$ increasing.

\subsubsection{Computational Complexity and Communication Load}
We focus on the computational complexity and communication load of updating $\boldsymbol{W}$ until convergence of the three algorithms. Suppose that the AGD, ADMM and CEASE algorithms stop in $q_1$, $q_2$ and $q_3$ iterations, respectively. Computing the gradients on $C$ machines is $O(C(mnr + mr^2))$, and then the complexity of AGD is $O(C(mnr + mr^2))$. AGD needs to collect the gradient and then broadcast the updated $\boldsymbol{W}$, so the communication load is $2q_1mr$. One iteration of ADMM involves matrix multiplication, soft thresholding and computation of the inverse of matrices of size $r\times r$, and thus has the complexity $O(C(mnr + mr^2 + r^3))$. The total complexity is $O(q_2C(mnr + mr^2 + r^3))$. ADMM only collects and broadcasts a matrix of $m\times r$. So, the communication load of ADMM is $2q_2 mr$. CEASE computes the $\boldsymbol{W}_c$ by FISTA. Suppose FISTA stops in $t$ iterations, and then the complexity of CEASE is  $O(tq_3C(mnr + mr^2))$. CEASE broadcasts and collects both the gradients and $\boldsymbol{W}_c$, so the communication load is $4q_3mr$. Table  \ref{table:compleixy} summarizes the complexity and the communication load of updating $\boldsymbol{W}$ for all the tree algorithms. Both ADMM and CEASE introduce the auxiliary variable $\boldsymbol{W}_c$ on the node machines to reduce the communication load, and they have to pay the extra computational cost.
\begin{table}[t]
    \caption{Complexity and Communication Load of Updating $\boldsymbol{W}$}
    \label{table:compleixy}
    \centering
    \begin{tabular}{ccc}
        \toprule
        & Complexity & Communication load \\
        \midrule
        AGD & $O(q_1C(mnr + mr^2))$& $2q_1mr$\\
        ADMM & $O(q_2C(mnr + mr^2 + r^3))$& $2q_2mr$\\
        CEASE &$O(tq_3C(mnr + mr^2))$ & $4q_3mr$\\
        \bottomrule
    \end{tabular}
\end{table}

\begin{figure*}
    \centering
    \includegraphics[width=0.8\linewidth]{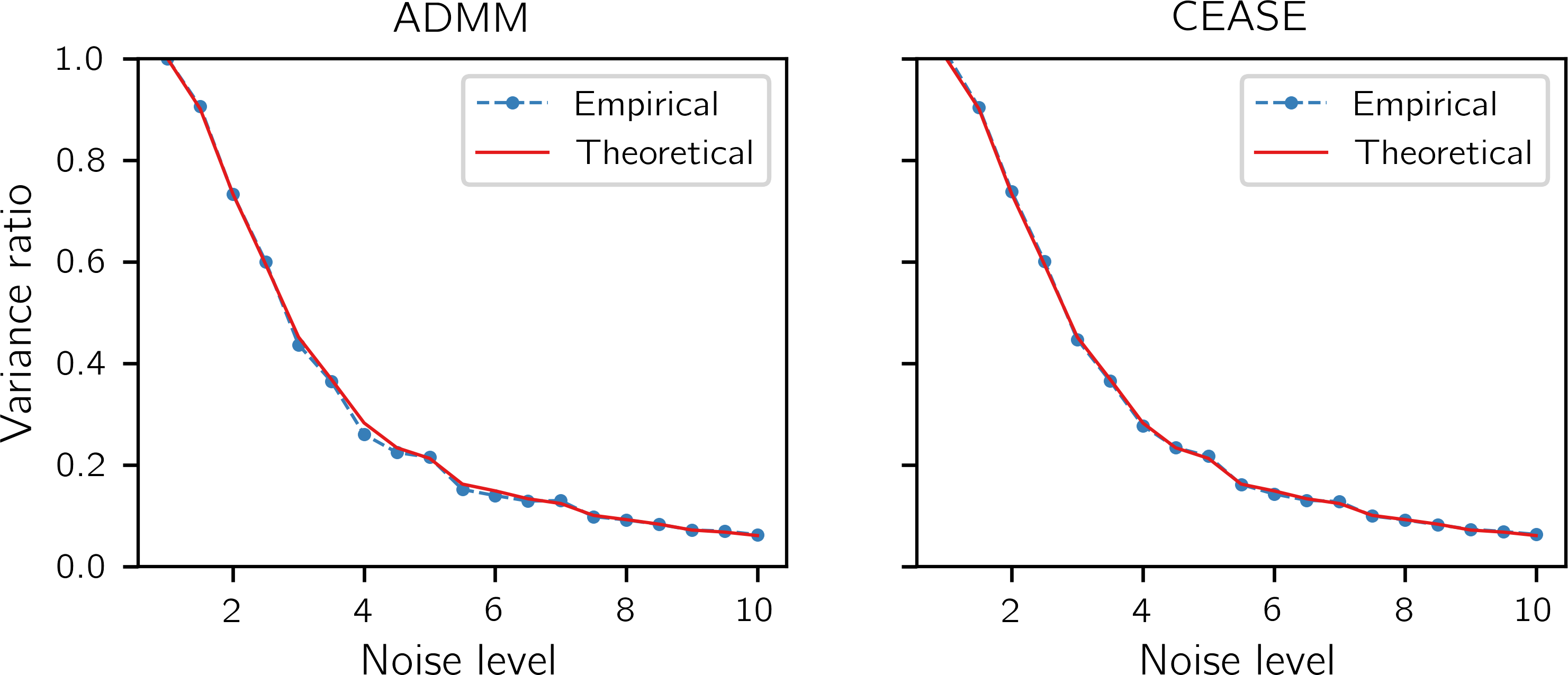}
    \caption{The variance ratio $\var(\tilde{\boldsymbol{W}}) / \var(\bar{\boldsymbol{W}})$ on a series of synthetic datasets $\{\boldsymbol{X}_c\}_{c=1}^5$, where the noise level $\sigma_c = 1, c\in [4]$, and $\sigma_5$ increases from 1 to 10.}
    \label{fig:var_reduction}
\end{figure*}

\begin{figure*}
    \centering
    \includegraphics[width=0.8\linewidth]{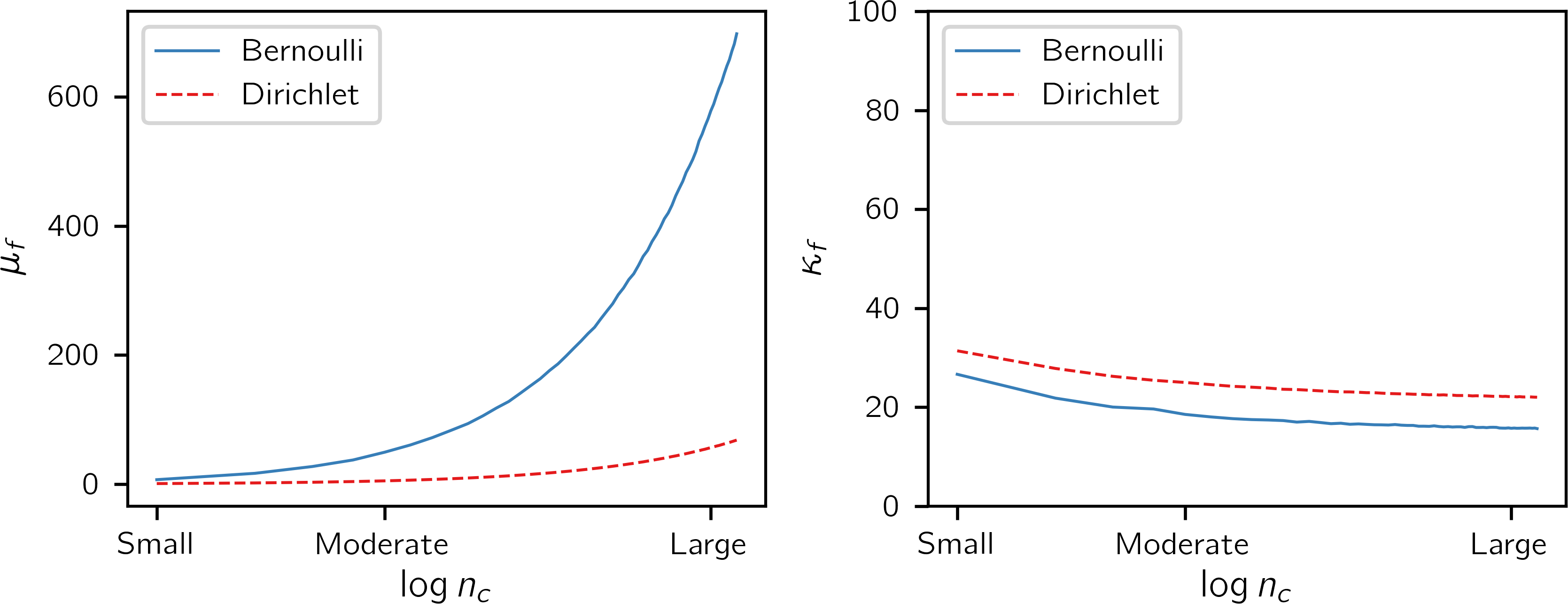}
    \caption{The $\log n_c$ versus the largest eigenvalue $\sigma_{\text{max}}$ and the condition number $\kappa$ of $\sum_{c=1}^C \boldsymbol{H}_c\boldsymbol{H}_c^T$, respectively. $\boldsymbol{H}_c$ are drawn from the Bernoulli and Dirichlet distributions respectively. $n_c$ ranges from 100 to 6000, and $\sigma_{\text{max}}$ and $\kappa$ are the average of 100 times for a given $n_c$.}
     \label{fig:hc_eigen}
\end{figure*}

\begin{figure*}[htbp]
    \centering
    \includegraphics[width=\textwidth]{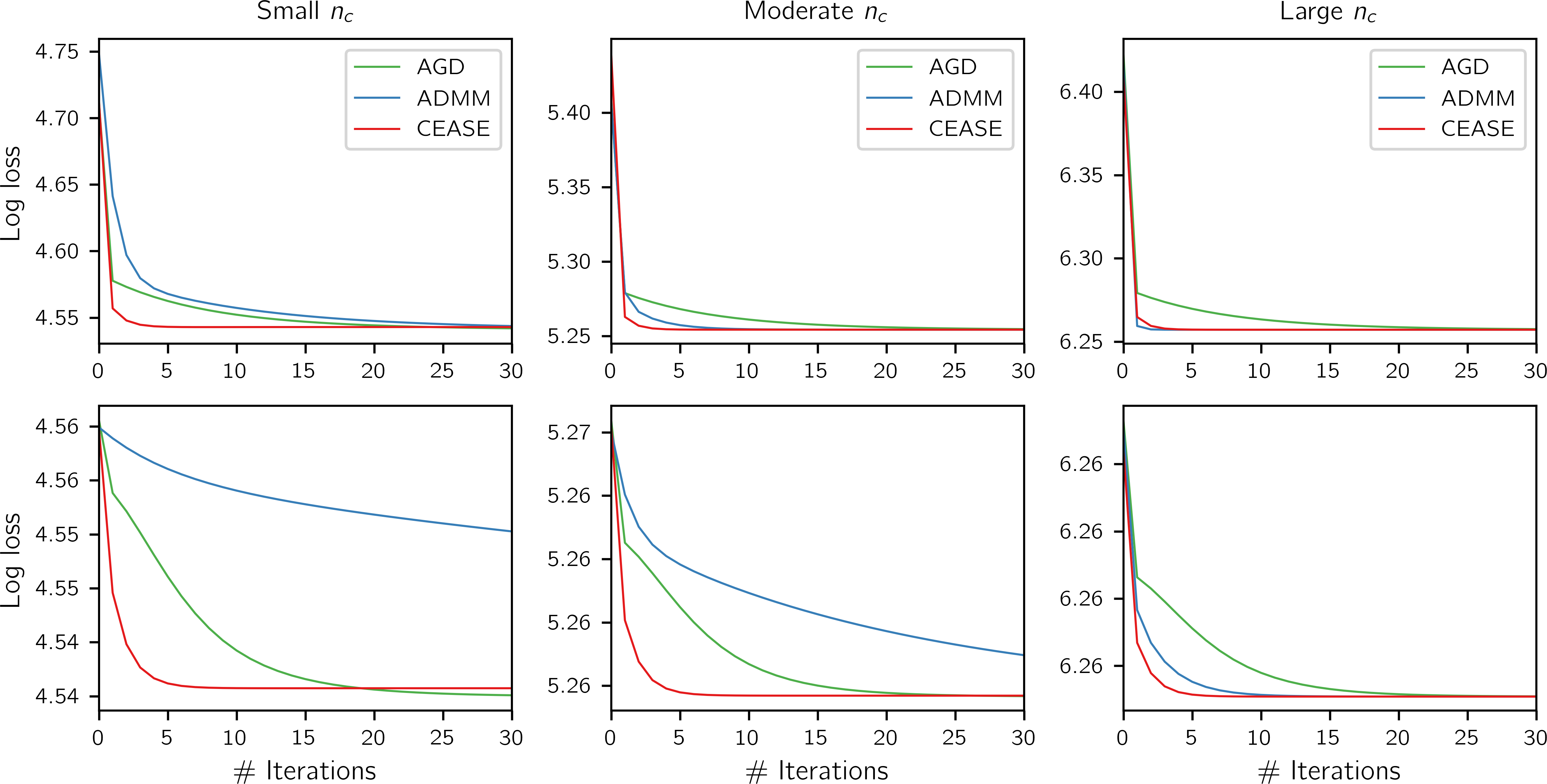}
    \caption{Log loss versus the number of iterations of the AGD, ADMM and CEASE algorithms respectively for updating $\boldsymbol{W}$. Top: results on synthetic data A; bottom: results on synthetic data B}
    \label{fig:obj_vals}
\end{figure*}

\section{Experimental Results}
We first evaluated the three algorithms on synthetic data to verify the theoretical analysis of updating $\boldsymbol{W}$.
Then we applied the proposed methods to real-world datasets for clustering and compared them with the distributed k-means and Scalable-NMF.
The synthetic experiments were performed on a desktop computer with a 2GHz Intel Xeon E5-2683 v3 CPU, a GTX 1080 GPU card, 16GB memory, and the real-world experiments were performed on a small spark cluster with six machines (one central processor and the rest are node machines).
Each machine is equipped with a Intel i7 CPU and 16GB memory.
We allow each physical machine runs two virtual machines at most. Therefore, there are at most 10 node machines.
The source code is available at https://github.com/zhanglabtools/dbmd.

\subsection{Synthetic Experiments}
We generate the basis matrix $\boldsymbol{W} \in R^{m\times r}$ inspired by \cite{Wu2016}
\begin{equation}
 w_{ik} =
 \begin{cases}
 a,  &  1 + (k-1)(l-coh)\leq i \leq l + (k-1)(l-coh) \quad \\
 &  k \in [r]\\
 0,  & \text{otherwise}
 \end{cases}
\end{equation}
where $a$ is a constant, $l$ denotes the number of non-zero entries in each column of $\boldsymbol{W}$,  and $coh$ denotes the length of coherence between basis $\boldsymbol{w}_{i-1, \cdot}$ and $\boldsymbol{w}_{i\cdot}$.
We generated the coefficient matrices $\{\boldsymbol{H}_c\}_{c=1}^C$ in two different ways:
1) draw entries of  $\boldsymbol{H}_c$ from the Bernoulli distribution $(h_{kj})_c \sim \text{B}(1, p)$. We set the last entries of all zero columns of $\boldsymbol{H}_c$ to 1, and then we normalize $\boldsymbol{H}_c$ such that the sum of column equals one;
2) draw columns of $\boldsymbol{H}_c$ from the Dirichlet distribution with a parameter $\boldsymbol{\alpha}$. Then the observed data matrices $\{\boldsymbol{X}_c\}^{C}_{c=1}$ are generated by  $\boldsymbol{X}_c = \boldsymbol{W}\boldsymbol{H}_c + \boldsymbol{E}_c$, where $(e_{ij})_c \sim N(0, \sigma_c^2)$. We suppose that $\boldsymbol{H}_c$ is known in this subsection.

\par To verify the effectiveness of the weighted averages, we generated a series of datasets $\{\boldsymbol{X}_c\}_{c=1}^5$ with $a=1.5$, $l=20$, $n_c=100$, $r=10$ and $coh = 2$. $\{\boldsymbol{H}_c\}_{c=1}^5$ were drawn from the Bernoulli distribution with $p=0.1$.
We set the noise level $\sigma_c = 1, c\in [4]$ and $\sigma_5$ ranges from 1 to 10. Consequently, $\boldsymbol{X}_c \in R^{182 \times 100}$.
We applied the ADMM with $\rho = 50$ and CEASE to the $\{\boldsymbol{X}_c\}_{c=1}^5$ with known $\boldsymbol{H}_c$.
The theoretical variance ratio is given in Theorem \ref{thn:variance reduction}.
We also computed the empirical variance ratio by repeating the experiments for 100 times.
The result confirms the correctness of our theoretical analysis in Theorem \ref{thn:variance reduction}.
The empirical variance ratio fits the theoretical line well for both the ADMM and CEASE algorithms (Fig. \ref{fig:var_reduction}).
 With $\sigma_c$ increasing, the estimated variance of $\boldsymbol{W}$ by the weighted average is smaller than that of the simple average (variance ratio approaches 0).
 Therefore, the plug-in weighted average can significantly reduce the variance of the estimated $\boldsymbol{W}$ when the heterogeneous noise exists.

\par We then investigated the convergence behaviors of the proposed methods with small, moderate, and large $n_c$ on node machines.
To facilitate the comparison, we generated the first synthetic data $A$ $\{\boldsymbol{X}_c\}_{c=1}^5$ with $a=1.5$, $l=20$, $r=20$ , $coh = 2$ $\sigma_c=1$ and $n_c =$ 100, 500, 5000, respectively.
$\boldsymbol{H}_c$ were drawn from the Bernoulli distribution with $p = 1 /20$.
Synthetic data $A$ contains 3 datasets with different $n_c$.
We generated another synthetic data $B$ with the same parameters, but $\boldsymbol{H}_c$ were drawn from the Dirichlet distribution with $\boldsymbol{\alpha} = \boldsymbol{1}$.
Given $n_c$, $\sigma_{\text{max}}(\sum_{c=1}^C \boldsymbol{H}_c\boldsymbol{H}_c^T)$ of $\boldsymbol{H}_c$ drawn from the Bernoulli distribution is greater than that of $\boldsymbol{H}_c$ drawn from the Dirichlet distribution; $\sigma_{\text{max}}(\sum_{c=1}^C \boldsymbol{H}_c\boldsymbol{H}_c^T)$ grows linearly with $n_c$ increasing (Fig. \ref{fig:hc_eigen}, left).
But the condition number $\kappa(\sum_{c=1}^C \boldsymbol{H}_c\boldsymbol{H}_c^T)$ doesn't change a lot with $n_c$ increasing (Fig. \ref{fig:hc_eigen}, right).

\par We plot the curves of the number of iterations versus the objective function values (Fig. \ref{fig:obj_vals}).
The convergence behaviors confirm our analysis:
1) ADMM and CEASE converge faster when $n_c$ gets larger, because $\sigma_{\text{max}}(\sum_{c=1}^C \boldsymbol{H}_c\boldsymbol{H}_c^T)$ gets larger when $n_c$ increases.
2) ADMM converges faster when $\boldsymbol{H}_c$ are drawn from the Bernoulli distribution (Fig. \ref{fig:obj_vals}, top row), because $\sigma_{\text{max}}(\sum_{c=1}^C \boldsymbol{H}_c\boldsymbol{H}_c^T)$ is larger than that drawn from the Dirichlet distribution. It is the same for CEASE.
3) The convergence of AGD takes around the same number of steps for small, moderate, and large size of $n_c$.
Unlike ADMM and CEASE, the convergence speed of AGD doesn't change with $n_c$ increasing.
There are also some other interesting observations:
1) The gradient-enhanced loss of CEASE accelerates its convergence. CEASE converges faster when $n_c$ is small and moderate.
2) ADMM is very slow when $n_c$ is small. But when $n_c$ is sufficiently large, ADMM may take fewer steps than AGD.
3) Because we can compute the Lipschitz constant $L$ of $\nabla f(\boldsymbol{W})$ directly, and $1/L$ is the largest step size.
Therefore, AGD is quite fast on this problem, and it converges within 30 steps. The experimental results show that CEASE reduces the number iterations regardless of $n_c$.
But ADMM reduces the number iterations when $n_c$ is sufficiently large.

\subsection{Real-World Experiments}
We further applied the proposed algorithms to real-world large-scale datasets for clustering.
\par \textbf{Datasets}. We downloaded three datasets, including CoverType, KDD99 and MINIST from the UCI Machine Learning Repository (https://archive.ics.uci.edu/ml/datasets.php). CoverType is the forest cover type data.
KDD-99 is the network connection data. Different types of network connections were given in the data.
KDD-99 contains 23 classes.
We removed the types the number of occurrences is below 100, and eleven classes are remaining.
MNIST is the handwritten digits data.
Each image is of size $28 \times 28$ and thus can be represented by a $784$-dimensional vector.
We used a subset of MNIST.
Table \ref{table:stats} summarizes the statistics of the used datasets.

\par \textbf{Experiment settings}. For all experiments, we set $\rho=300$, $\gamma=0.001$, $\boldsymbol{\alpha} = 1$.
We set the $L_1$-norm regularizer parameter $\lambda=2000, 4000, 500$ for CoverType, KDD-99, MNIST, respectively.
We always set $r$ equals the number of classes for convenience.
To facilitate the comparison of time costs of ADMM and CEASE, we stop the procedure of updating of $\boldsymbol{W}$ at the $(k+1)$-th iteration, if $||\boldsymbol{W}^{k+1}-\boldsymbol{W}^{k}||\leq \norm{\boldsymbol{W}^0}_F \times 10^{-2}$. AGD will stop immediately, because the step size $1 / L_f$ is typically very small.
Therefore, we ensured that the AGD algorithm iterates at least 30 rounds. The column of $H_c$ indicates the membership of the corresponding instance.
We assigned an instance $(\boldsymbol{x}_{\cdot j})_c$ to the class corresponding to the largest entry of the $(\boldsymbol{h}_{\cdot j})_c$.
We then evaluated the performance of the clustering by accuracy. The true classes and the predicted ones were matched by the Hungarian algorithm \cite{Kuhn1955}.
\begin{figure*}
    \centering
    \includegraphics[width=0.90\linewidth]{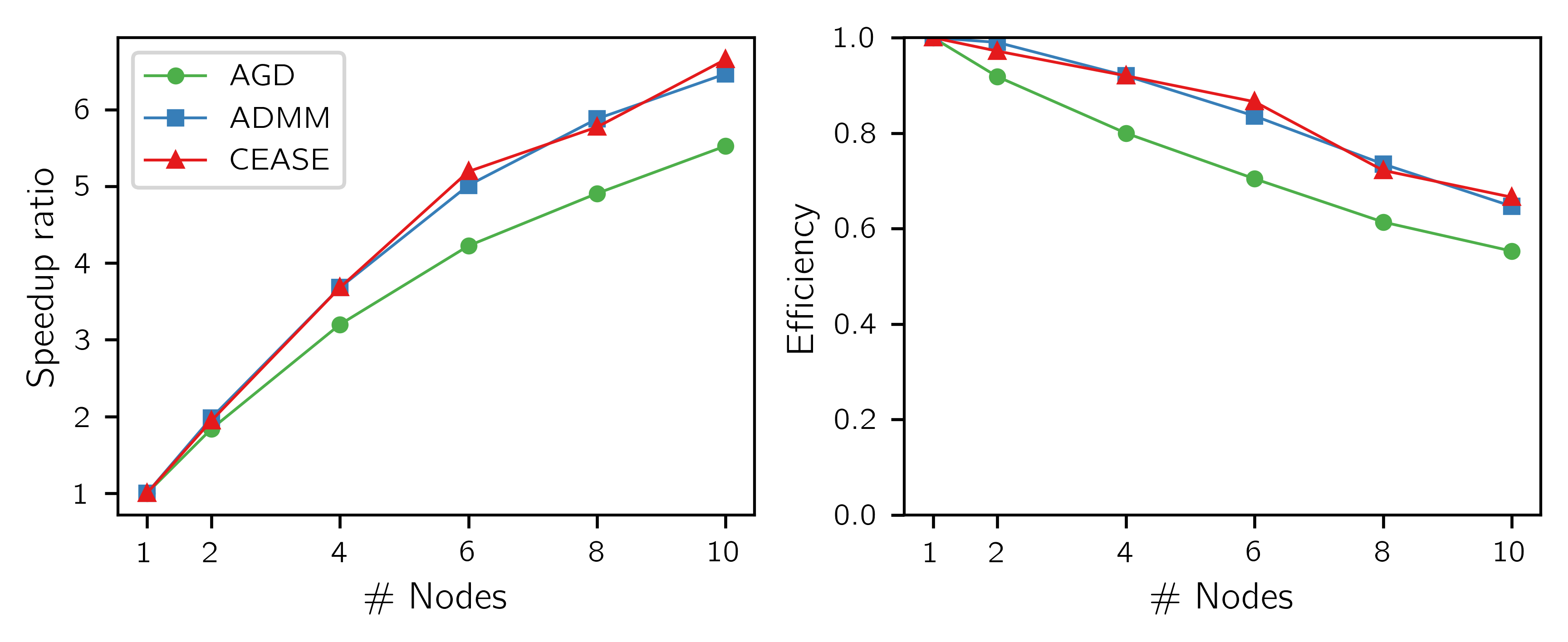}
    \caption{Scalability of the proposed methods on CoverType dataset.
        The speedup ratio of using $n$  nodes is defined by the ratio of the running time of single machine versus that of $n$ node machines.
        The efficiency is the reciprocal of the speedup ratio.}
    \label{fig:running_time.png}
\end{figure*}

\begin{table}[t]
    \caption{Summary of the Three Datasets}
    \centering
    \begin{tabular}{cccc}
        \toprule
        Dataset& \# Instances & \# Features & \# Classes \\
        \midrule
        CoverType &  \texttildelow 540,000 & 54 & 7\\
        KDD-99 & \texttildelow 4,900,000   & 41 & 11 \\
        MNIST  &  \texttildelow 400,000  & 784 & 10 \\
        \bottomrule
    \end{tabular}
\label{table:stats}
\end{table}

\par We first evaluated the scalability of our algorithms and applied them to the CoverType dataset with the number of node machines ranging from 1 to 10.
The proposed algorithms scale up well with the number of node machines increasing (Fig. \ref{fig:running_time.png}, left).
We can find that the efficiency decay of ADMM and CEASE is smaller than that of AGD (Fig. \ref{fig:running_time.png}, right).
Because AGD needs more rounds of communications for updating $\boldsymbol{W}$.

\par To compare the performance of clustering, we applied our algorithms to the three datasets and compared them with two distributed clustering algorithms based on matrix decomposition, including the Spark implementations of scalable k-means++ \cite{Meng2016} and Scalable-NMF \cite{Gittens2016} (Table \ref{table:performance}).
Generally speaking, our methods achieve competitive or superior performances compared to scalable k-means++ and scalable-NMF. Scalable-NMF has a poor performance on the MNIST data, while the accuracy of the proposed methods are acceptable.
The separable assumption of Scalable-NMF may be violated on real-world data. Scalable-NMF is very faster because it only requires one round iteration.
Our methods involve alternatively updating $\boldsymbol{W}$ and $\boldsymbol{H}_c$, which is computationally expensive.
But the time cost is still acceptable.
Compared with CEASE, AGD is still faster.
Because the local network latency is low, and AGD has a quadratic convergence rate.
CEASE reduces the communication load by paying more computational cost on the node machines. Among the three proposed algorithms, ADMM is the fastest one because it has closed-form solutions at each step of updating $\boldsymbol{W}$ and the $n_c$ is sufficiently large.

\par To verify the robustness of the proposed methods to the noise, we created a series of noisy MNIST data.
Specifically, we added Gaussian noises of $\sigma_1 = 0.1$, $\sigma_2=1.0$ to 20\% and 60\% of the instances of MNIST data. We then added Gaussian noises of the standard deviation $\sigma_3$ varying from $0$ to $9.5$ in step of 0.5 to the remaining 20\% instances.
In consequence, we generated the semi-synthetic MNIST datasets under 20 different noise settings.
We then applied the scalable k-means++ and the proposed methods to them.
Scalable-NMF is omitted for its poor performance.
Scalable k-means++ suffers from the increasing noise, and its performance drops down sharply and is very unstable (Fig. \ref{fig:noisy_mnist.png}).
On the contrary, all of the proposed methods show a mild performance decline while the noisy level increases.
It implies that the Bayesian priors reduce the risk of overfitting to the highly noisy data.
Moreover, the performance of ADMM and CEASE are slightly better than that of AGD when the noisy level is sufficiently large ($\sigma_3 >6.5$).
 It reminds us that considering the heterogeneity of the noise also contributes to the robustness of the model.
\begin{table*}[!t]
    \centering
    \begin{threeparttable}
        \caption{Clustering Performances of Different Methods on the Three Datasets}
        \begin{tabular}{lcccccc}
            \toprule
            &  \multicolumn{2}{c}{CoverType} & \multicolumn{2}{c}{KDD-99} & \multicolumn{2}{c}{MNIST} \\
            \cmidrule(lr){2-3} \cmidrule(lr){4-5}  \cmidrule(lr){6-7}
            & Accuracy & Time (s)  & Accuracy&  Time (s) & Accuracy& Time (s) \\
            \midrule
            Scalable-NMF         &   33.79 & 2.14 & 79.35 & 13.90  & 20.55& 25.77\\
            Scalable k-means++  &   29.71(3.44)  & 12.55  & 72.62 (6.73) & 42.54 & 47.96(1.75) & 78.56\\
            DBMD-AGD &               42.20(0.32) &139.43 & 89.44(0.62) &1279.29 & 43.70(2.73) &456.50\\
            DBMD-ADMM &              42.58(0.36) &156.34& 89.76(0.04) &1045.10 & 43.32(3.17) &375.87\\
            DBMD-CEASE &             42.32(0.32) &157.23& 89.45(0.45) &1517.39 & 43.38(3.18) &799.92\\
            \bottomrule
        \end{tabular}
        \begin{tablenotes}
            \small
            \item The means and the standard deviations of 5 runs are shown here. The standard deviations of the Scalable-NMF are not reported, because it is a deterministic algorithm. Time is in seconds (s).
        \end{tablenotes}
        \label{table:performance}
    \end{threeparttable}
\end{table*}

\begin{figure}
	\centering
 	\includegraphics[width=.95\linewidth]{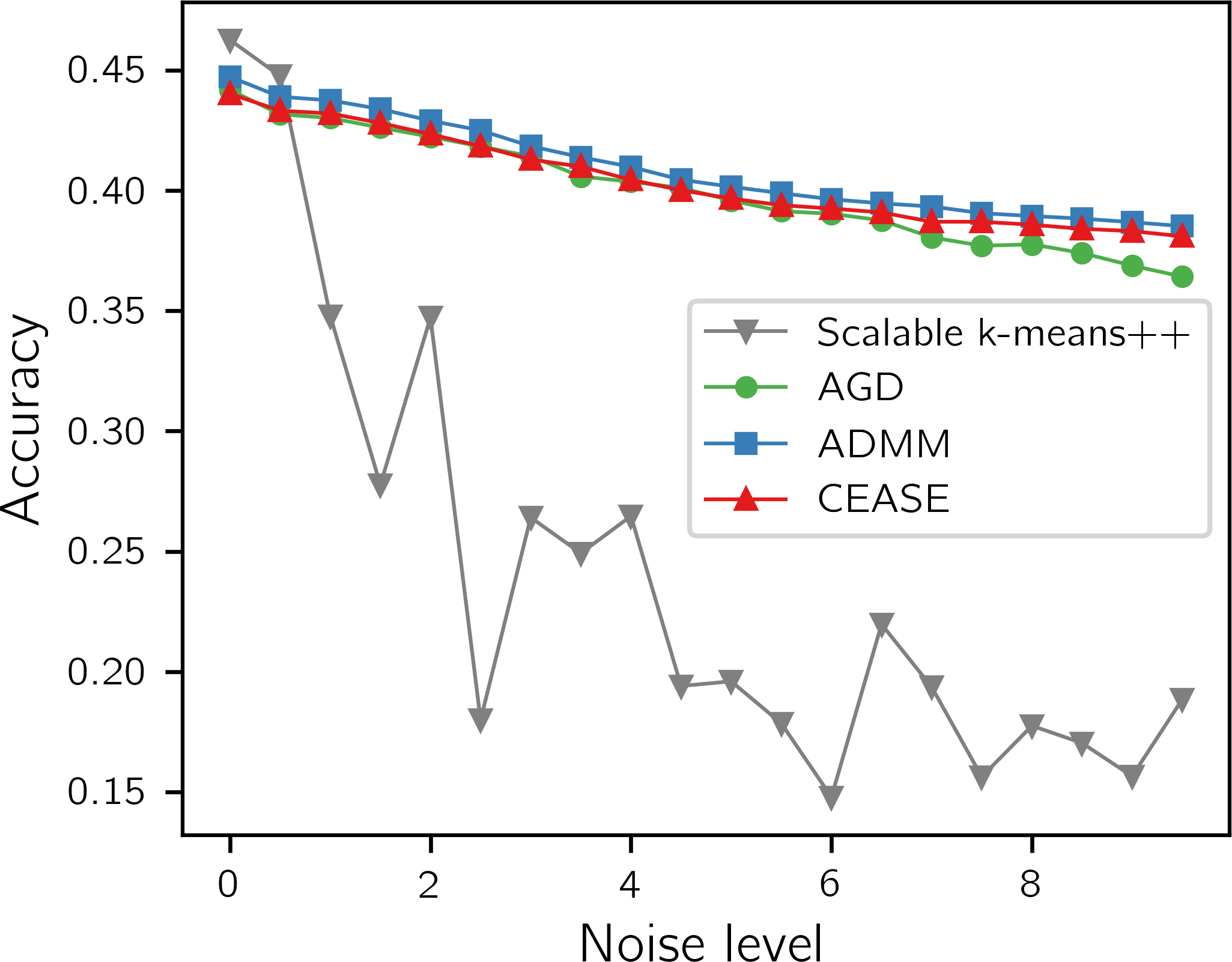}
	\caption{The clustering performance on the noisy MNIST data. The noise level of 20\% instances increasing from 0 to 9.5. The accuracy is the average of 10 runs.}
	\label{fig:noisy_mnist.png}
\end{figure}

\section{Discussion and Conclusion}
\par We proposed a distributed Bayesian matrix decomposition model for big data mining and clustering. Three distributed strategies (i.e., AGD, ADMM and CEASE) were adopted to solve it. The convergence rates of AGD and CEASE depend on different structural parameters (i.e., $\mu_f$ and $\kappa_f$) and thus have different behaviors. In short, CEASE converges faster with the number of instances on each node machine increasing, but the convergence rate of AGD doesn't change much. Empirically, ADMM also converges faster with the number of instances growing. To tackle the heterogeneous noise in the data, we propose an optimal plug-in weighted average scheme that  significantly reduces the variance of the estimation. The proposed algorithms scale up well. The real-world experiments demonstrate that the proposed algorithms achieve superior or competitive performance. Both the Bayesian prior and the weighted average strategies reduce the influence of the highly noisy data.

\par There are several questions worth investigating in future studies. First, the concept of the weighted average can be generalized to other algorithms. Second, we assume that the transposition of the data matrix is tall-and-skinny, which is limited. It is commonplace for modern applications that the data matrix is fat and tall, i.e., the numbers of rows and columns are both vast. Finally, we observed that ADMM converges faster when $n_c$ is larger. The convergence rate of this algorithm needs to be further investigated.


%

\appendices

\section*{Acknowledgment}
This work has been supported by the National Natural Science Foundation of China [11661141019, 61621003, 61762089];
National Ten Thousand Talent Program for Young Top-notch Talents; National Key Research and Development Program of China [2017YFC0908405]; CAS Frontier Science Research Key Project for Top Young Scientist [QYZDB-SSW-SYS008].


\ifCLASSOPTIONcaptionsoff
  \newpage
\fi

\bibliographystyle{IEEEtran}
\bibliography{IEEEabrv,reference}

\begin{thebibliography}{10}
\providecommand{\url}[1]{#1}
\csname url@samestyle\endcsname
\providecommand{\newblock}{\relax}
\providecommand{\bibinfo}[2]{#2}
\providecommand{\BIBentrySTDinterwordspacing}{\spaceskip=0pt\relax}
\providecommand{\BIBentryALTinterwordstretchfactor}{4}
\providecommand{\BIBentryALTinterwordspacing}{\spaceskip=\fontdimen2\font plus
\BIBentryALTinterwordstretchfactor\fontdimen3\font minus
  \fontdimen4\font\relax}
\providecommand{\BIBforeignlanguage}[2]{{%
\expandafter\ifx\csname l@#1\endcsname\relax
\typeout{** WARNING: IEEEtran.bst: No hyphenation pattern has been}%
\typeout{** loaded for the language `#1'. Using the pattern for}%
\typeout{** the default language instead.}%
\else
\language=\csname l@#1\endcsname
\fi
#2}}
\providecommand{\BIBdecl}{\relax}
\BIBdecl

\bibitem{Pearson1901}
K.~Pearson, ``{ LIII. no lines and planes of closest fit to systems of points
  in space },'' \emph{London, Edinburgh, Dublin Philos. Mag. J. Sci.}, vol.~2,
  no.~11, pp. 559--572, nov 1901.

\bibitem{Bishop2006}
C.~Bishop, \emph{{Pattern recognition and machine learning}}.\hskip 1em plus
  0.5em minus 0.4em\relax Springer, 2006.

\bibitem{Shen2014}
Y.~Shen, Z.~Wen, and Y.~Zhang, ``{Augmented Lagrangian alternating direction
  method for matrix separation based on low-rank factorization},'' \emph{Optim.
  Methods Softw.}, vol.~29, no.~2, pp. 239--263, mar 2014.

\bibitem{Min2019}
W.~Min, J.~Liu, and S.~Zhang, ``Group-sparse svd models via $ l\_1 $-and $ l\_0
  $-norm penalties and their applications in biological data,'' \emph{IEEE
  Trans. Knowl. Data Eng.}, 2019.

\bibitem{Paatero1994}
P.~Paatero and U.~Tapper, ``{Positive matrix factorization: A non-negative
  factor model with optimal utilization of error estimates of data values},''
  \emph{Environmetrics}, vol.~5, no.~2, pp. 111--126, jun 1994.

\bibitem{Lee1999}
D.~D. Lee and H.~S. Seung, ``{Learning the parts of objects by non-negative
  matrix factorization},'' \emph{Nature}, vol. 401, no. 6755, pp. 788--791, oct
  1999.

\bibitem{Hoyer2002}
P.~O. Hoyer, ``{Non-negative sparse coding},'' in \emph{Neural Networks Signal
  Process. - Proc. IEEE Work.}, vol. 2002-January, 2002, pp. 557--565.

\bibitem{Kim2007}
H.~Kim and H.~Park, ``{Sparse non-negative matrix factorizations via
  alternating non-negativity-constrained least squares for microarray data
  analysis},'' \emph{Bioinformatics}, vol.~23, no.~12, pp. 1495--1502, jun
  2007.

\bibitem{Cai2008}
D.~Cai, X.~He, X.~Wu, and J.~Han, ``{Non-negative matrix factorization on
  manifold},'' in \emph{Proc. IEEE Int. Conf. Data Mining}, 2008, pp. 63--72.

\bibitem{Cai2011}
D.~Cai, X.~He, J.~Han, and T.~S. Huang, ``{Graph regularized nonnegative matrix
  factorization for data representation},'' \emph{IEEE Trans. Pattern Anal.
  Mach. Intell.}, vol.~33, no.~8, pp. 1548--1560, 2011.

\bibitem{Zhang2011}
S.~Zhang, Q.~Li, J.~Liu, and X.~J. Zhou, ``A novel computational framework for
  simultaneous integration of multiple types of genomic data to identify
  microrna-gene regulatory modules,'' \emph{Bioinformatics}, vol.~27, no.~13,
  pp. i401--i409, 2011.

\bibitem{Tipping1999}
M.~E. Tipping and C.~M. Bishop, ``{Probabilistic principal component
  analysis},'' \emph{J. R. Stat. Soc. Ser. B (Statistical Methodol.)}, vol.~61,
  no.~3, pp. 611--622, aug 1999.

\bibitem{bishop1999bayesian}
C.~M. Bishop, ``{Bayesian PCA},'' in \emph{Adv. Neural Inf. Process. Syst.},
  1999, pp. 382--388.

\bibitem{Collins2001}
M.~Collins, S.~Dasgupta, and R.~E. Schapire, ``{A generalization of principal
  component analysis to the exponential family},'' in \emph{Adv. Neural Inf.
  Process. Syst.}, 2001, pp. 617--624.

\bibitem{mohamed2009bayesian}
S.~Mohamed, Z.~Ghahramani, and K.~A. Heller, ``{Bayesian exponential family
  PCA},'' in \emph{Adv. Neural Inf. Process. Syst.}, 2009, pp. 1089--1096.

\bibitem{Li2013}
J.~Li and D.~Tao, ``{Simple exponential family PCA},'' \emph{IEEE Trans. Neural
  Networks Learn. Syst.}, vol.~24, no.~3, pp. 485--497, 2013.

\bibitem{Welling2006}
M.~Welling and K.~Kurihara, ``{Bayesian k-means as a "maximization-expectation"
  algorithm},'' in \emph{Proc. SIAM Int. Conf. Data Min.}, vol. 2006, 2006, pp.
  474--478.

\bibitem{Schmidt2009}
M.~N. Schmidt, O.~Winther, and L.~K. Hansen, ``{Bayesian non-negative matrix
  factorization},'' in \emph{Int. Conf. Indep. Compon. Anal. Signal Sep.}\hskip
  1em plus 0.5em minus 0.4em\relax Springer, 2009, pp. 540--547.

\bibitem{Cemgil2009}
A.~T. Cemgil, ``Bayesian inference for nonnegative matrix factorisation
  models,'' \emph{Comput. Intell. Neurosci.}, vol. 2009, pp. 1--17, 2009.

\bibitem{Salakhutdinov2008}
R.~Salakhutdinov and A.~Mnih, ``{Bayesian probabilistic matrix factorization
  using Markov chain Monte Carlo},'' in \emph{Proc. Int. Conf. Mach.
  Learn.}\hskip 1em plus 0.5em minus 0.4em\relax ACM, 2008, pp. 880--887.

\bibitem{Saddiki2015}
H.~Saddiki, J.~McAuliffe, and P.~Flaherty, ``{GLAD: a mixed-membership model
  for heterogeneous tumor subtype classification},'' \emph{Bioinformatics},
  vol.~31, no.~2, pp. 225--232, jan 2015.

\bibitem{Xu2013}
C.~Xu, D.~Tao, and C.~Xu, ``A survey on multi-view learning,'' \emph{arXiv
  preprint arXiv:1304.5634}, 2013.

\bibitem{Zhang2012}
S.~Zhang, C.-C. Liu, W.~Li, H.~Shen, P.~W. Laird, and X.~J. Zhou, ``{Discovery
  of multi-dimensional modules by integrative analysis of cancer genomic
  data},'' \emph{Nucleic Acids Res.}, vol.~40, no.~19, pp. 9379--9391, 2012.

\bibitem{Jing2012}
L.~Jing, C.~Zhang, and M.~K. Ng, ``{SNMFCA: supervised NMF-based image
  classification and annotation},'' \emph{IEEE Trans. Image Process.}, vol.~21,
  no.~11, pp. 4508--4521, 2012.

\bibitem{Liu2013}
J.~Liu, C.~Wang, J.~Gao, and J.~Han, ``{Multi-view clustering via joint
  nonnegative matrix factorization},'' in \emph{Proc. SIAM Int. Conf. Data
  Min.}\hskip 1em plus 0.5em minus 0.4em\relax SIAM, 2013, pp. 252--260.

\bibitem{Zhang2019b}
L.~Zhang and S.~Zhang, ``{A General Joint Matrix Factorization Framework for
  Data Integration and its Systematic Algorithmic Exploration},'' \emph{IEEE
  Trans. Fuzzy Syst.}, 2019.

\bibitem{Zhang2019a}
------, ``Learning common and specific patterns from data of multiple
  interrelated biological scenarios with matrix factorization,'' \emph{Nucleic
  Acids Res.}, vol.~47, no.~13, pp. 6606--6617, 2019.

\bibitem{Zhang2019}
C.~Zhang and S.~Zhang, ``Bayesian joint matrix decomposition for data
  integration with heterogeneous noise,'' \emph{IEEE Trans. Pattern Anal. Mach.
  Intell.}, pp. 1--1, 2019.

\bibitem{Jin2006}
R.~Jin, A.~Goswami, and G.~Agrawal, ``{Fast and exact out-of-core and
  distributed k-means clustering},'' \emph{Knowl. Inf. Syst.}, vol.~10, no.~1,
  pp. 17--40, 2006.

\bibitem{Bahmani2012}
B.~Bahmani, B.~Moseley, A.~Vattani, R.~Kumar, and S.~Vassilvitskii, ``{Scalable
  k-means++},'' \emph{Proc. VLDB Endow.}, vol.~5, no.~7, pp. 622--633, 2012.

\bibitem{Yu2014}
Z.-Q. Yu, X.-J. Shi, L.~Yan, and W.-J. Li, ``{Distributed stochastic ADMM for
  matrix factorization},'' in \emph{Proc. ACM Int. Conf. Conf. Inf. Knowl.
  Manag.}\hskip 1em plus 0.5em minus 0.4em\relax ACM, 2014, pp. 1259--1268.

\bibitem{Ahn2015}
S.~Ahn, A.~Korattikara, N.~Liu, S.~Rajan, and M.~Welling, ``{Large-scale
  distributed Bayesian matrix factorization using stochastic gradient MCMC},''
  in \emph{Proc. ACM SIGKDD Int. Conf. Knowl. Discov. Data Min.}\hskip 1em plus
  0.5em minus 0.4em\relax ACM, 2015, pp. 9--18.

\bibitem{Qin2019}
X.~Qin, P.~Blomstedt, E.~Lepp{\"{a}}aho, P.~Parviainen, S.~Kaski, J.~Davis,
  E.~Fromont, D.~Greene, and B.~B. {Bringmann Xiangju Qin}, ``{Distributed
  Bayesian matrix factorization with limited communication},'' \emph{Mach.
  Learn.}, vol. 108, pp. 1805--1830, 2019.

\bibitem{Liu2010}
C.~Liu, H.-c. Yang, J.~Fan, L.-W. He, and Y.-M. Wang, ``{Distributed
  nonnegative matrix factorization for web-scale dyadic data analysis on
  mapreduce},'' in \emph{Proc. Int. Conf. World Wide Web}.\hskip 1em plus 0.5em
  minus 0.4em\relax ACM, 2010, pp. 681--690.

\bibitem{Benson2014}
A.~R. Benson, J.~D. Lee, B.~Rajwa, and D.~F. Gleich, ``{Scalable methods for
  nonnegative matrix factorizations of near-separable tall-and-skinny
  matrices},'' in \emph{Adv. Neural Inf. Process. Syst.}, Z.~Ghahramani,
  M.~Welling, C.~Cortes, N.~D. Lawrence, and K.~Q. Weinberger, Eds.\hskip 1em
  plus 0.5em minus 0.4em\relax Curran Associates, Inc., 2014, pp. 945--953.

\bibitem{Zdunek2017}
R.~Zdunek and K.~Fonal, ``{Distributed nonnegative matrix factorization with
  HALS algorithm on MapReduce},'' in \emph{Lect. Notes Comput. Sci.}, vol.
  10393 LNCS.\hskip 1em plus 0.5em minus 0.4em\relax Springer Verlag, 2017, pp.
  211--222.

\bibitem{Arthur2007}
D.~Arthur and S.~Vassilvitskii, ``{k-means++: The advantages of careful
  seeding},'' in \emph{Proc. Annu. ACM-SIAM Symp. Discret. Algorithms}.\hskip
  1em plus 0.5em minus 0.4em\relax Society for Industrial and Applied
  Mathematics, 2007, pp. 1027--1035.

\bibitem{Donoho2004}
D.~Donoho and V.~Stodden, ``{When does non-negative matrix factorization give a
  correct decomposition into parts?}'' in \emph{Adv. Neural Inf. Process.
  Syst.}, S.~Thrun, L.~K. Saul, and B.~Sch{\"{o}}lkopf, Eds.\hskip 1em plus
  0.5em minus 0.4em\relax MIT Press, 2004, pp. 1141--1148.

\bibitem{Beck2009}
A.~Beck and M.~Teboulle, ``A fast iterative shrinkage-thresholding algorithm
  for linear inverse problems,'' \emph{SIAM J. Imag. Sciences}, vol.~2, no.~1,
  pp. 183--202, 2009.

\bibitem{Lan2018}
G.~Lan, S.~Lee, and Y.~Zhou, ``{Communication-efficient algorithms for
  decentralized and stochastic optimization},'' \emph{Math. Program.}, 2018.

\bibitem{Jordan2019}
M.~I. Jordan, J.~D. Lee, and Y.~Yang, ``{Communication-efficient distributed
  statistical inference},'' \emph{J. Am. Stat. Assoc.}, vol. 114, no. 526, pp.
  668--681, 2019.

\bibitem{fan2019communication}
J.~Fan, Y.~Guo, and K.~Wang, ``Communication-efficient accurate statistical
  estimation,'' \emph{arXiv preprint arXiv:1906.04870}, 2019.

\bibitem{Shamir2014}
O.~Shamir, N.~Srebro, and T.~Zhang, ``{Communication-efficient distributed
  optimization using an approximate newton-type method},'' in \emph{Int. Conf.
  Mach. Learn.}, 2014, pp. 1000--1008.

\bibitem{boyd2011distributed}
S.~Boyd, N.~Parikh, E.~Chu, B.~Peleato, J.~Eckstein \emph{et~al.},
  ``Distributed optimization and statistical learning via the alternating
  direction method of multipliers,'' \emph{Foundation News and
  Trends{\textregistered} in Machine learning}, vol.~3, no.~1, pp. 1--122,
  2011.

\bibitem{Tao2016}
S.~Tao, D.~Boley, and S.~Zhang, ``{Local linear convergence of ISTA and FISTA
  on the LASSO problem},'' \emph{SIAM J. Optim.}, vol.~26, no.~1, pp. 313--336,
  2016.

\bibitem{Wu2016}
S.~Wu, A.~Joseph, A.~S. Hammonds, S.~E. Celniker, B.~Yu, and E.~Frise,
  ``Stability-driven nonnegative matrix factorization to interpret spatial gene
  expression and build local gene networks,'' \emph{Proc. Natl. Acad. Sci.},
  vol. 113, no.~16, pp. 4290--4295, 2016.

\bibitem{Kuhn1955}
H.~W. Kuhn, ``The hungarian method for the assignment problem,'' \emph{Nav.
  Res. Logist. Q.}, vol.~2, no. 1--2, pp. 83--97, 1955.

\bibitem{Meng2016}
X.~Meng, J.~Bradley, B.~Yavuz, E.~Sparks, S.~Venkataraman, D.~Liu, J.~Freeman,
  D.~B. Tsai, M.~Amde, S.~Owen, and Others, ``{Mllib: machine learning in
  Apache Spark},'' \emph{J. Mach. Learn. Res.}, vol.~17, no.~1, pp. 1235--1241,
  2016.

\bibitem{Gittens2016}
A.~Gittens, A.~Devarakonda, E.~Racah, M.~Ringenburg, L.~Gerhardt, J.~Kottalam,
  J.~Liu, K.~Maschhoff, S.~Canon, J.~Chhugani, and Others, ``{Matrix
  factorizations at scale: a comparison of scientific data analytics in Spark
  and C++ MPI using three case studies},'' in \emph{IEEE Int. Conf. Big
  Data}.\hskip 1em plus 0.5em minus 0.4em\relax IEEE, 2016, pp. 204--213.

\end{thebibliography}

\end{document}